\documentclass[accepted]{uai2025}

\usepackage[l2tabu, orthodox]{nag}

\usepackage{natbib}
\usepackage[american]{babel}

\usepackage{amsmath}
\usepackage{amsthm}
\usepackage{amsfonts}
\usepackage{csquotes}
\usepackage{booktabs}
\usepackage{paralist}

\usepackage{bm}
\usepackage{bbm}        

\usepackage{array}
\usepackage{multirow}
\usepackage{tcolorbox}
\usepackage{mathrsfs}
\usepackage{tikz}
\usepackage{graphicx,comment,subcaption}
\usepackage{ifthen}
\usetikzlibrary{calc}
\usepackage[noend]{algpseudocode}
\usepackage{pgfplots}
\pgfplotsset{compat=1.15}
\usetikzlibrary{shapes,positioning,arrows,arrows.meta,calc,automata,matrix,fit,backgrounds}
\tikzstyle{state}+=[minimum size = 6mm, inner sep=0,outer sep=1]

\colorlet{disabled}{lightgray}
\tikzstyle{state}=[draw,rectangle,inner sep=5pt,rounded corners=2pt]
\tikzstyle{action}=[font=\small,inner sep=0pt,outer sep=3pt]
\tikzstyle{actionnode}=[circle,draw=black,fill=black,minimum size=1mm,inner sep=0,outer sep=0]
\tikzstyle{actionedge}=[draw,-]
\tikzstyle{prob}=[font=\scriptsize,inner sep=0pt,outer sep=1pt]
\tikzstyle{probedge}=[draw,->]
\tikzstyle{directedge}=[draw,->]
\tikzset{chainarrow/.tip={Stealth[length=3pt]}}
\tikzset{>=chainarrow}
\tikzset{inpolicy/.style={line width=1.3pt}}

\usepackage{xparse} 
\usepackage{mathtools}
\usepackage{environ}
\usepackage{marvosym}

\usepackage{algorithmicx,algorithm}
\usepackage[noend]{algpseudocode}

\usepackage{soul}
\usepackage{todonotes} 

\usepackage[capitalize]{cleveref}

\usepackage[switch]{lineno}

\newcommand{\B}{\mathcal{B}}

\newcommand{\C}{\mathcal{C}}
\newcommand{\calP}{\mathcal{P}}

\newcommand{\calO}{O}

\newcommand{\defas}{\coloneqq}

\newcommand{\ZZ}{\mathbb{Z}}

\newcommand{\val}{\text{val}}

\newcommand{\pot}{\text{pot}}
\newcommand{\appr}{\mathcal{A}}
\newcommand{\apprtwo}{\mathcal{A}_2}

\newcommand{\Vertices}{V} 
\newcommand{\vertex}{v} 
\newcommand{\otherver}{u} 
\newcommand{\Edges}{E}

\newcommand{\weight}{w}

\newcommand{\avg}{\text{Avg}}
\newcommand{\liminfavg}{\text{LimAvg}}

\newcommand{\run}{\omega}
\newcommand{\runs}{\Omega}
\newcommand{\policy}{\sigma}
\newcommand{\initpol}{\pi}
\newcommand{\finalpol}{\tau}
\newcommand{\policies}{\Sigma}
\newcommand{\wmax}{\bar{w}}
\newcommand{\vali}[1]{{w_{#1}}}

\theoremstyle{plain} 
\newtheorem{@theorem}{Theorem}[section]
\newenvironment{theorem}{\begin{@theorem}}{\end{@theorem}}

\newtheorem{lemma}[@theorem]{Lemma}

\newtheorem{conjecture}[@theorem]{Conjecture}

\theoremstyle{definition} 

\theoremstyle{remark} 

\title{Lower Bound on Howard Policy Iteration for Deterministic Markov Decision Processes}
\author[1]{\href{mailto:<ali.asadi@ista.ac.at>}{Ali Asadi}{}}
\author[1]{\href{mailto:<krishnendu.chatterjee@ist.ac.at>}{Krishnendu Chatterjee}{}}
\author[2]{\href{mailto:<jderaaij@fas.harvard.edu>}{Jakob de Raaij}{}}

\affil[1]{%
    Institute of Science and Technology Austria (ISTA)
}
\affil[2]{%
    Harvard University
}

\begin{document}
\maketitle

\begin{abstract}
    Deterministic Markov Decision Processes (DMDPs) are a mathematical framework for decision-making where the outcomes and future possible actions are deterministically determined by the current action taken. DMDPs can be viewed as a finite directed weighted graph, where in each step, the controller chooses an outgoing edge. An objective is a measurable function on runs (or infinite trajectories) of the DMDP, and the value for an objective is the maximal cumulative reward (or weight) that the controller can guarantee. We consider the classical mean-payoff (aka limit-average) objective, which is a basic and fundamental objective.
    
    Howard's policy iteration algorithm is a popular method for solving DMDPs with mean-payoff objectives. Although Howard's algorithm performs well in practice, as experimental studies suggested, 
    the best known upper bound is exponential and the current known lower bound is as follows: For the input size $I$, the algorithm requires $\widetilde{\Omega}(\sqrt{I})$ iterations, where $\widetilde{\Omega}$ hides the poly-logarithmic factors, i.e., the current lower bound on iterations is sub-linear with respect to the input size.  Our main result is an improved lower bound for this fundamental algorithm where we show that for the input size $I$, the algorithm requires $\widetilde{\Omega}(I)$ iterations. 
\end{abstract}

\section{Introduction}

\paragraph{Deterministic Markov Decision Processes.}
Deterministic Markov Decision Processes (DMDPs) \citep{Puterman94} are a mathematical framework for sequential decision-making where an agent interacts with a fully deterministic environment.  They are modeled as a graph, where in every step the controller chooses a successor vertex from the neighbors of the current vertex. This repeated process generates an infinite sequence of vertices (called a run). Policies for the controller provide the successor vertex choice at every vertex. A payoff function assigns a real value to every run. We consider a classical and well-studied function: the mean-payoff (or limit-average) payoff function~\citep{Puterman94,Filar12}. Every edge of the graph is assigned an integer weight, and the
payoff of a run is the long-run average of the weights of the run.

\paragraph{Applications.}
This formalism is particularly relevant in settings where system behaviors are fully known, such as controlled robotic environments or algorithmic planning tasks~\citep{blondel2000survey}. DMDPs also appear in formal verification and synthesis, where deterministic transitions allow for tractable analysis of safety, liveness, and temporal logic specifications~\citep{Baier08}. For example, in autonomous systems, DMDPs are a basic model for synthesizing controllers that guarantee correct behaviors~\citep{Alur15}.
Besides the practical applications, DMDPs can demonstrate fundamental computational limits, e.g., NP-hardness of optimal planning \citep{Littman97}. Furthermore, this model corresponds to classical directed weighted graphs, which have many applications such as network routing, travel planning, etc~\citep{Cormen22}.

\paragraph{Motivation.} One of the main algorithm in the area of planning and sequential decision-making is Howard's policy iteration~\citep{Howard60}. \citet{Dasdan04} compared various algorithms and showed that Howard's algorithm works well in practice as compared to other algorithms. Furthermore, lower and upper bounds for policy iteration algorithms have deep theoretical impact, e.g., lower bounds for policy iteration lead to lower bounds for pivoting approaches in linear programming~\citep{Friedmann11}. Hence better theoretical understanding of Howard's algorithm is an interesting problem. Our work aims at the theoretical understanding of this fundamental algorithm for DMDPs with mean-payoff objectives. We first recall the previous results from the literature.

\paragraph{Previous Work on Howard's Policy Iteration.}
Howard's policy iteration has been extensively studied for general MDPs (not necessarily DMDPs) with discounted-sum and mean-payoff objectives. For MDPs with mean-payoff objectives, the best-known upper bound is exponential~\citep{Puterman94}. \citet{Fearnley10}, inspired by the work of \citet{Friedmann09}, showed that Howard’s algorithm requires exponential time for a specific family of MDPs. For discounted-sum objectives, better upper bounds are known in special cases of the discount factor: \citet{Post15} established a strongly polynomial running time when the discount factor is constant or represented in unary, and \citet{hansen2013strategy} improved Post’s bounds and extended these results to 2-player settings. However, the exponential lower bound from \citet{Fearnley10} still holds for arbitrary discount factors. Related problems, such as MDPs with total-reward objectives (analogous to the stochastic shortest path), also exhibit exponential lower bounds for strategy iteration~\citep{Fearnley10}. While exponential lower bounds are established for stochastic models~\citep{Fearnley10}, game-theoretic models~\citep{Friedmann09}, and linear programming pivoting rules~\citep{Friedmann11}, identifying lower bounds specifically for Howard's policy iteration in simpler deterministic graph models has remained a compelling open question. Our contribution addresses this by presenting improved lower bounds for this simplest model, complementing the more general cases already explored in the literature.

\paragraph{Recent Results on Howard’s Policy Iteration.}
There are several recent advancements on Howard’s Policy Iteration in the literature, which highlight ongoing research in this direction. \citet{LoffSkomra24} demonstrated polynomial-time smoothed complexity for deterministic models (DMDPs and turn-based games with mean-payoff and discounted-sum objectives). However, \citet{ChristYannakakis23} presented a sub-exponential lower bound for the smoothed complexity of Howard’s algorithm for stochastic MDPs with mean-payoff objectives. Moreover, \citet{AsadiLICS24} recently improved complexity results for two-player turn-based discounted games with unary weights through a new analysis for Howard’s policy iteration.

\paragraph{Previous Work on DMDPs.}
DMDPs have been studied extensively in the literature~\citep{Arora12,Boone23,Castro20,Madani02,MadaniTZ09}. \citet{Karp78} presented an algorithm for solving DMDPs with mean-payoff objectives in $\calO(mn)$ time (where $m$ is the number of edges and $n$ is the number of vertices), while \citet{Young91} proposed an $\calO(mn + n^2 \log n)$‐time algorithm that often performs better in practice despite its slightly worse time complexity. Although Howard's policy iteration works well in practice~\citep{Dasdan04}, the known theoretical upper and lower bounds for DMDPs with mean-payoff objectives are as follows: The best-known upper bound is exponential~\citep{Puterman94}.
Moreover, a better parametric bound of $\calO(n^3 W)$ on the number of iterations can be obtained, since by \cite{Howard60}, the number of policy iteration steps is at most the number of value iteration steps, which by ~\citet{zwick1996complexity} is bounded by $\calO(n^3 W)$, where $W$ is the maximum absolute weight. 
\citet{hansen2010lower} presented a lower bound, giving DMDPs with $2n$ vertices, $m$ edges, and edge weights of $\calO(n^{{n^2}})$, on which the algorithm requires $m-n+1$ iterations to find an optimal policy.
For example, (a)~with $m=O(n)$, this result shows that on input size of $\widetilde{O}(n^3)$, the algorithm requires $\Omega(n)$ iterations, or (b)~with $m=O(n^2)$, this shows that on input size of $\widetilde{O}(n^4)$, the algorithm requires $\Omega(n^2)$ iterations. In particular, given the input description of an DMDP with $I$ bits, the results shows that the lower bound on iterations is $\widetilde{\Omega}(\sqrt{I})$. In computer science establishing and improving lower bounds are challenging, and whether this lower bound can be improved is a fundamental problem which we address in this work.

\paragraph{Our Contributions.}
The above motivates the study of Howard's policy iteration for DMDPs with mean-payoff objective. In this work, we construct a family of DMDPs with $2n$ vertices, $\calO(n^2)$ edges, and edge weights of $\calO(n^2)$, on which Howard's algorithm requires $\Omega(n^2)$ iterations to find an optimal policy. Hence, the improved lower bound is as follows. Given the input description of an DMDP with $I$ bits, the required number of iterations is $\widetilde{\Omega}(I)$. Table~\ref{table:results-sum} summarizes the results.
	
	\begin{table*}[t]
		\centering
		\caption{%
			Comparison of lower bounds for Howard's policy iteration. $|\Vertices|$, $|\Edges|$, and $W$ correspond to the number of vertices, number of edges, and maximum absolute weight, respectively.
		}
		\label{table:results-sum}
		\begin{tabular}{|c|c|c|c|c|c|}
			\hline
			  & $|\Vertices|$ & $|\Edges|$ & $W$ & Size & \# Iterations\\ 
                \hline
                \hline
                \citet{hansen2010lower} & $2n$ & $m$ & $\calO(n^{n^2})$ & $\calO(mn^2\log n)$ & $m-n+1$\\
                \hline
                Ours & $2n$ & $\calO(n^2)$ & $\calO(n^2)$ & $\calO(n^2\log n)$ & $\Omega(n^2)$\\
                \hline
		\end{tabular}
	\end{table*}

\paragraph{Significance.} As compared to the work of~\citet{hansen2010lower}, the significance of our result is twofold. First, with respect to the input size of $I$, we improve the lower bound from $\widetilde{\Omega}(\sqrt{I})$ to $\widetilde{\Omega}(I)$. Second, there is an important implication with respect to the $\calO(n^3W)$ parametric bound by \citet{zwick1996complexity}. For the family of examples considered in~\citet{hansen2010lower}, the best known upper bound is exponential as the weights are exponential. Thus, the examples of~\citet{hansen2010lower} belong to a class where the upper bound is exponential  and a sub-linear lower bound is presented. In contrast, since the weights are polynomial for our class of examples, the upper bound on the number of iterations is polynomial (namely, quadratic with respect to the input size) and we present an almost-linear lower bound.

\paragraph{Technical Novelty.}
In the examples of \citet{hansen2010lower}, Howard's policy iteration goes through $\Theta(n^2)$ "good" cycles and due to the structure of the graph only ever finds a slightly better cycle. Each cycle weight is determined by a "key" edge with largest weight. To achieve the lower bound for policy iteration, the algorithm requires the next key edge weight to differ from the previous one by a multiplicative factor of $n$ (or the cycle length, which can be up to $n$), leading to exponential edge weights. In contrast, our examples avoid this, inspired by the results of~\citet{Friedmann09,Fearnley10}, using a similar concept to their "deceleration lane". Our DMDPs only have $n$ "good" cycles that do not overlap edgewise. Thus, each cycle only needs to have a weight of $1$ more than the previous one. The "deceleration lane" technique forces the algorithm to perform $\Omega(i)$ iterations to find the $i$th cycle after having found the $i-1$th cycle. This structure ensures that (a)~Howard's algorithm finds the cycles in the right order and (b) performs $\Omega(n^2)$ iterations to find an optimal policy.
\section{Preliminaries}
\label{Section: Preliminaries}

We present standard notations and definitions related to deterministic markov decision processes.

\paragraph{Deterministic Markov Decision Processes.} 
A deterministic markov decision process~(DMDP) is a finite directed weighted graph $P = (\Vertices, \Edges, \weight)$ consisting of  
\begin{compactitem}
    \item the set of vertices $\Vertices$, of size $n$; 
    \item the set of edges $\Edges \subseteq \Vertices \times \Vertices$, of size $m$, such that for all $\vertex \in \Vertices$, the set $\Edges(\vertex) \defas \{ \otherver \mid (\vertex, \otherver) \in \Edges \}$ is non-empty; and
    \item the weight function $\weight: \Edges \to \ZZ$ that assigns a weight $\weight(\vertex, \otherver)$ for all edges $(\vertex, \otherver) \in \Edges$.
\end{compactitem}
We denote the largest absolute weight by $W \defas \max \{ |\weight(\vertex, \otherver)|  \mid  (\vertex, \otherver) \in \Edges \}$. The size of $P$ is defined as $|P| \defas n + m +  \sum_{(\vertex, \otherver) \in \Edges} \lceil \log_2 |\weight(\vertex, \otherver)| \rceil$. The vertices are indexed and have an ordering.

\paragraph{Steps and Runs.} 
Given an initial vertex $\vertex_0 \in \Vertices$, the process proceeds as follows.
In each step, the controller chooses the next vertex from the set $\Edges(\vertex)$.
A \emph{run} is an infinite sequence of vertices $\run = \langle \vertex_0, \vertex_1, \ldots \rangle$ where for every step $t \ge 0$, the vertex $\vertex_{t+1} \in \Edges(\vertex_t)$. We denote by $\runs$ 
the set of all runs, and by $\policies_\vertex$ the set of all runs $\run = \langle \vertex_0, \vertex_1, \ldots \rangle$
where $\vertex_0 = \vertex$.

\paragraph{Mean-payoff Objectives.} 
An objective is a measurable function that assigns a real number to all runs. For a run $\run = \langle \vertex_0, \vertex_1, \ldots \rangle$, the average for $t$ steps is $\avg_t(\run) \defas \frac{1}{t} \sum_{i = 0}^{t-1} \weight(\vertex_i, \vertex_{i+1})$. The $\liminf$ average is $\liminfavg(\run) \defas \liminf_{t \to \infty} \avg_t(\run)$. The objective of controller is to maximize the $\liminf$ average of the run.

\paragraph{Positional Policies.}
Policies are recipes that specify how to choose the next vertex.
A {\em positional}  policy $\policy \colon \Vertices \to \Vertices$ for the controller is a policy which chooses
a vertex~$\policy(\vertex) \in \Edges(\vertex)$ whenever the run visits vertex~$\vertex$. 
We denote by $\policies^P$ the set of all positional policies.
In general, policies can depend on past history and not only the current vertex. 
However, for mean-payoff objectives, positional policies are as powerful as general policies~\citep{Puterman94}.
Hence, in the sequel, every policy is positional.

\paragraph{Runs Given Policies in DMDPs.}
We define $P^\policy$ as the restricted DMDP where the controller follows the policy $\policy$. Note that once the controller has fixed their policy, we obtain a graph where each vertex has exactly one outgoing edge.
Given an initial vertex $\vertex$, we obtain a run $P_\vertex^{\policy} = \langle \vertex_0, \vertex_1, \ldots \rangle$ such that $\vertex_0 = \vertex$, and for any step $t \ge 0$, $\vertex_{t+1} = \policy(\vertex_t)$. The obtained run $P_\vertex^{\policy}$ is a {\em lasso-shaped}  run that consists in a finite cycle-free path $\calP \defas \langle \vertex_0, \ldots, \vertex_{p} \rangle$ followed by a simple cycle $\C \defas \langle \vertex_p, \ldots, \vertex_{p + c - 1} \rangle$
repeated forever, where $\vertex_p$ is the {\em head} of the cycle (the vertex with the least index in the cycle). The mean-payoff of the policy $\policy$ is defined as
\[
    \val^\policy(\vertex) \defas \liminfavg(P_\vertex^{\policy}) = \frac{1}{c} \sum_{i=0}^{c-1} \weight(\vertex_{p + i}, \vertex_{p + i + 1}) \,.
\]
We define the potential function as
\[
    \pot^\policy(\vertex) \defas \sum_{i=0}^{p-1} \left (\weight(\vertex_i, \vertex_{i+1}) - \val^\policy(\vertex) \right )\,.
\]
In words, the payoff $\val^{\policy}(\vertex)$ is the mean-payoff the controller obtains, in case they follow the policy $\policy$, and the potential $\pot^\policy(\vertex)$ is the relative distance from $\vertex$ to $\vertex_p$, where the weight of each edge is subtracted by the mean-payoff.

\paragraph{Value and Optimal Policies.}
The mean-payoff value for a vertex $\vertex$ is defined as
    $\val(\vertex) \defas \max_{\policy \in \policies^P} \val^{\policy}(\vertex)$.
A policy~$\policy$ for the controller is \emph{optimal} for mean-payoff objectives if, for all vertices $\vertex \in \Vertices$, we have $\val^\policy(\vertex) = \val(\vertex)$. 

\paragraph{Bellman Operator.} Given a policy $\policy$, we define the {\em appraisal} of an edge $(\vertex, \otherver)$ as a tuple
\[
    \appr^\policy(\vertex, \otherver) \defas (\val^\policy(\otherver), \weight(\vertex, \otherver) - \val^\policy(\otherver) + \pot^\policy(\otherver) ) \,. 
\]
The Bellman operator, which is an operator from $\policies^P$ to $\policies^P$, is defined as
\[\B(\policy)(\vertex) \defas \arg \max_{\otherver \in \Edges(\vertex)} \appr^\policy(\vertex, \otherver)\,.\]
The appraisals are compared lexicographically, and ties are resolved by first favoring $\otherver=\policy(\vertex)$, then vertices with the least index. For increased legibility, we will refer to the second term of the appraisal as \[\apprtwo^\policy(\vertex, \otherver) \defas \weight(\vertex, \otherver) - \val^\policy(\otherver) + \pot^\policy(\otherver)\,.\]

\paragraph{Howard's Policy Iteration.} Howard's policy iteration is a classical algorithm for computing the optimal policies in DMDPs with mean-payoff objectives. The algorithm starts with an arbitrary policy $\policy_0$. 
In each iteration, the algorithm locally improves the current policy: Starting with $\policy_k$ at iteration $k$, the algorithm computes the payoff and the potential of the policy $\policy_k$. Using these, it updates the policy using the Bellman operator defined above by setting $\policy_{k+1}=\B(\policy_k)$. The algorithm terminates if $\policy_{k+1}=\policy_k$, meaning no update to the policy was made. The correctness of Howard's algorithm is shown in \citet{derman1970finite,Puterman94}.

\section{Overview of Results}
A natural question on Howard's policy iteration is that how many iteration it takes to find an optimal policy. In the following, we state a long-standing conjecture  on the upper bound for Howard's policy iteration. 
\begin{conjecture}[{\cite[Conjecture 6.1.1]{hansen2012worst}}]
    The number of iterations performed by Howard’s algorithm, when applied to a DMDP, is at most the number of edges.
\end{conjecture}
\citet{hansen2010lower} constructed a family of DMDPs with $n$ vertices and $m$ edges on which Howard's algorithm performs $m - n +1$ iterations. However, the size of DMDPs is $\calO(mn^2 \log n)$ due to exponential weights. In this work, we present a family of DMDPs with $2n$ vertices and $\calO(n^2)$ edges on which Howard's algorithm performs $\Omega(n^2)$ iterations to find an optimal policy. The weights are bounded by $\calO(n^2)$. Hence, the size of our DMDPs is $\calO(n^2 \log n)$, which improves the dependency on the number of edges from linear to constant. Our main result is stated as follows.

    \begin{theorem}[Main Result]
    \label{thm:main-result}
        Let $n$ be a positive integer. There exists a DMDP with $2n$ vertices, $\frac{3n^2 + n}{2}$ edges, and size of $O(n^2\log n)$ on which Howard's algorithm performs $\frac{n^2 + 7n - 6}{2}$ iterations to find an optimal policy.
    \end{theorem}

\section{Improved Lower Bound}
In this section, we construct a family of DMDPs with $2n$ vertices of size $\calO(n^2 \log n)$ on which Howard's algorithm performs $\Omega(n^2)$ iterations. 

\subsection{DMDP Construction}
Given a positive integer $n$, we construct a DMDP $P_n = (\Vertices_n, \Edges_n, \weight_n$). We denote the set of vertices by $\Vertices_n = \{b_1, \ldots, b_n, t_1, \ldots, t_n  \}$. The ordering of vertices is $(t_1, b_1, \ldots, b_n, t_2, \ldots, t_n)$.  We denote the set of edges by
\begin{align*}
    \Edges_n 
    &\defas  \{(b_i, b_j) \mid 1 \le j < i \le n \} 
    \\  & \,\,\cup  \{(b_i, t_j) \mid 1 \le i, j \le n\}
    \\  &\,\,\cup \{ (t_i, b_j) \mid 1 \le j \le i \le n \} 
    \\  &\,\,\cup \{ (t_i, t_j) \mid 0 \le j \le i \le n \}\,.
\end{align*}
We now define the weight function as
\[
     w_n(\vertex, \otherver) \defas \begin{cases}
        (n+1)^2  & \vertex = b_i \land \otherver = b_j \\& \text{for all } 1 \le j < i \le n\\
        (n+1)^2 & \vertex = t_i \land \otherver = b_j \\& \text{for all } 1 \le j \le i \le n\\
        0  & \vertex = b_i \land \otherver = t_j \\& \text{for all } 1 \le i, j \le n\\
        0 & \vertex = t_i \land \otherver = t_j \\& \text{for all } 1 \le j < i \le n\\
        n(n+1) + i & \vertex = t_i \land \otherver = t_i \\& \text{for all } 1 \le i \le n\\
        \end{cases}
\]
\Cref{fig:dmdp} illustrates an example of the DMDP with $n=3$.

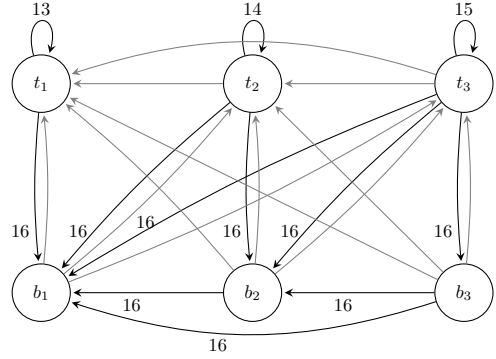
\begin{figure}[t]
\centering
\resizebox{0.8\linewidth}{!}{%
    \begin{tikzpicture}[->,>=stealth,shorten >=1pt,auto,node distance=2.5cm,scale=0.7,transform shape]

\tikzstyle{state}=[circle,draw,minimum size=1.1cm]
\def\weightpos{0.8};
\def\bend{5};
\def\bendsingle{0};
\def\zerolabel{\,};

\node[state] (t1) at (4,0) {$t_1$};
\node[state] (t2) at (8,0) {$t_2$};
\node[state] (t3) at (12,0) {$t_3$};

\node[state] (b1) at (4,-4) {$b_1$};
\node[state] (b2) at (8,-4) {$b_2$};
\node[state] (b3) at (12,-4) {$b_3$};

\path[->] (t1) edge[loop above] node {$13$} (t1);
\path[->] (t2) edge[loop above] node {$14$} (t2);
\path[->] (t3) edge[loop above] node {$15$} (t3);

\path[->] (t2) edge[pos=0.75, above, gray] node { } (t1);
\path[->] (t3) edge[pos=0.75, above, gray] node { } (t2);
\path[->] (t3) edge[bend right=18, pos=0.75, above, gray] node { } (t1);

\path[->] (t1) edge[bend right=\bend, pos=\weightpos] node[left] {$16$} (b1);
\path[->] (t2) edge[bend right=\bend, pos=\weightpos] node[left] {$16$} (b2);
\path[->] (t2) edge[bend right=\bend, pos=\weightpos] node[left] {$16$} (b1);
\path[->] (t3) edge[bend right=\bend, pos=0.75] node[left] {$16$} (b1);
\path[->] (t3) edge[bend right=\bend, pos=\weightpos] node[left] {$16$} (b2);
\path[->] (t3) edge[bend right=\bend, pos=\weightpos] node[left] {$16$} (b3);

\path[->] (b1) edge[bend right=\bend,pos=\weightpos, gray] node[right] {\zerolabel} (t1);
\path[->] (b1) edge[bend right=\bend,pos=\weightpos, gray] node[right] {\zerolabel} (t2);
\path[->] (b1) edge[bend right=\bend,pos=\weightpos, gray] node[right] {\zerolabel} (t3);
\path[->] (b2) edge[bend right=\bend,pos=\weightpos, gray] node[right] {\zerolabel} (t1);
\path[->] (b2) edge[bend right=\bend,pos=\weightpos, gray] node[right] {\zerolabel} (t2);
\path[->] (b2) edge[bend right=\bend,pos=\weightpos, gray] node[right] {\zerolabel} (t3);
\path[->] (b3) edge[bend right=\bendsingle,pos=\weightpos, gray] node[right] {\zerolabel} (t1);
\path[->] (b3) edge[bend right=\bendsingle,pos=\weightpos, gray] node[right] {\zerolabel} (t2);
\path[->] (b3) edge[bend right=\bend,pos=\weightpos, gray] node[right] {\zerolabel} (t3);

\path[->] (b2) edge[pos=0.6, below] node {$16$} (b1);
\path[->] (b3) edge[pos=0.6, below] node {$16$} (b2);
\path[->] (b3) edge[bend left = 18, pos=0.6, below] node {$16$} (b1);

\end{tikzpicture}
}
\caption{Our running example with $n=3$. Unlabeled (gray) edges have weight $0$.}
\label{fig:dmdp}
\end{figure}

\subsection{Policies}
We describe three families of policies that appear in Howard's algorithm. We first define the family of policies $\initpol_{i}$ for $1 \le i \le n$ as 
\[
    \initpol_{i}(\vertex) \defas \begin{cases}
        b_k & \vertex = t_k \qquad \text{for all } 1 \le k \le i-2\\
        t_{k} & \vertex = t_k \qquad \text{for all } k \in \{i-1, i\}\\
        t_{i} & \text{otherwise}
    \end{cases}
\]
We now define another family of policies $\policy_{i,j}$ for $1 \le i \le n$ and $1 \le j \le i+1$  as
\[
    \policy_{i,j}(\vertex) \defas \begin{cases}
        t_i & \vertex = t_i\\
        b_k & \vertex = t_k \qquad \text{for all } (1 \le k \le j \land k \neq i)\\
        & \phantom{\vertex = t_k} \qquad \text{or }  (k\le i-2 \land j=1)\\
        t_i & \vertex = b_1\\
        b_{k-1} & \vertex = b_k \qquad \text{for all } 2 \le k \leq j\\
        b_j & \text{otherwise}
    \end{cases}
\]
Finally, we define the last family of policies $\finalpol_{i}$ for $1 \le i \le n$ as
\[
    \finalpol_{i}(\vertex) \defas \begin{cases}
        t_k & \vertex = t_k \qquad \text{for all } k \in \{i, i+1\}\\
        b_k & \vertex = t_k \qquad \text{for all } 1 \le k < i\\
        t_i & \vertex = b_1\\
        b_{k-1} & \vertex = b_k \qquad \text{for all } 2 \le k \leq i+2\\
        b_{i+2} & \text{otherwise}
    \end{cases}
\]
For more intuition, the policies $\initpol_{2}$, $\policy_{2,1}$, $\policy_{2,3}$, and $\finalpol_{2}$ over our running example are illustrated in \Cref{fig:policy-families}. An illustration of all policies for the running example can be found in \Cref{sec:policy-seq}. Moreover, the outline of policies that appear in the general DMDP $P_n$ is illustrated in \Cref{sec:general-pol}.

\begin{figure*}[t]
    \centering
    \begin{subfigure}{0.4\textwidth}
        \centering
        \resizebox{0.8\linewidth}{!}{%
            \begin{tikzpicture}[->,>=stealth,shorten >=1pt,auto,node distance=2.5cm,scale=0.7,transform shape]

\tikzstyle{state}=[circle,draw,minimum size=1.1cm]
\def\weightpos{0.8};
\def\bend{5};
\def\bendsingle{0};
\def\zerolabel{\,};

\node[state] (t1) at (4,0) {$t_1$};
\node[state] (t2) at (8,0) {$t_2$};
\node[state] (t3) at (12,0) {$t_3$};

\node[state] (b1) at (4,-4) {$b_1$};
\node[state] (b2) at (8,-4) {$b_2$};
\node[state] (b3) at (12,-4) {$b_3$};

\path[->] (t1) edge[loop above, inpolicy] node {$\mathbf{13}$} (t1);
\path[->] (t2) edge[loop above, inpolicy] node {$\mathbf{14}$} (t2);
\path[->] (t3) edge[loop above] node {$15$} (t3);

\path[->] (t2) edge[pos=0.75, above, gray] node { } (t1);
\path[->] (t3) edge[pos=0.75, above, inpolicy] node {$\mathbf{0}$} (t2);
\path[->] (t3) edge[bend right=18, pos=0.75, above, gray] node { } (t1);

\path[->] (t1) edge[bend right=\bend, pos=\weightpos] node[left] {$16$} (b1);
\path[->] (t2) edge[bend right=\bend, pos=\weightpos] node[left] {$16$} (b2);
\path[->] (t2) edge[bend right=\bend, pos=\weightpos] node[left] {$16$} (b1);
\path[->] (t3) edge[bend right=\bend, pos=0.75] node[left] {$16$} (b1);
\path[->] (t3) edge[bend right=\bend, pos=\weightpos] node[left] {$16$} (b2);
\path[->] (t3) edge[bend right=\bend, pos=\weightpos] node[left] {$16$} (b3);

\path[->] (b1) edge[bend right=\bend,pos=\weightpos, gray] node[right] {\zerolabel} (t1);
\path[->] (b1) edge[bend right=\bend,pos=0.85, inpolicy] node[right] {$\mathbf{0}$} (t2);
\path[->] (b1) edge[bend right=\bend,pos=\weightpos, gray] node[right] {\zerolabel} (t3);
\path[->] (b2) edge[bend right=\bend,pos=\weightpos, gray] node[right] {\zerolabel} (t1);
\path[->] (b2) edge[bend right=\bend,pos=0.8, inpolicy] node[right] {$\mathbf{0}$} (t2);
\path[->] (b2) edge[bend right=\bend,pos=\weightpos, gray] node[right] {\zerolabel} (t3);
\path[->] (b3) edge[bend right=\bendsingle,pos=\weightpos, gray] node[right] {\zerolabel} (t1);
\path[->] (b3) edge[bend right=\bendsingle,pos=0.9, inpolicy] node[right] {$\mathbf{0}$} (t2);
\path[->] (b3) edge[bend right=\bend,pos=\weightpos, gray] node[right] {\zerolabel} (t3);

\path[->] (b2) edge[pos=0.6, below] node {$16$} (b1);
\path[->] (b3) edge[pos=0.6, below] node {$16$} (b2);
\path[->] (b3) edge[bend left = 18, pos=0.6, below] node {$16$} (b1);

\end{tikzpicture}%
        }
        \caption{Policy $\initpol_2$}
    \end{subfigure}
    \hfill
    \begin{subfigure}{0.4\textwidth}
        \centering
        \resizebox{0.8\linewidth}{!}{%
            \begin{tikzpicture}[->,>=stealth,shorten >=1pt,auto,node distance=2.5cm,scale=0.7,transform shape]

\tikzstyle{state}=[circle,draw,minimum size=1.1cm]
\def\weightpos{0.8};
\def\bend{5};
\def\bendsingle{0};
\def\zerolabel{\,};

\node[state] (t1) at (4,0) {$t_1$};
\node[state] (t2) at (8,0) {$t_2$};
\node[state] (t3) at (12,0) {$t_3$};

\node[state] (b1) at (4,-4) {$b_1$};
\node[state] (b2) at (8,-4) {$b_2$};
\node[state] (b3) at (12,-4) {$b_3$};

\path[->] (t1) edge[loop above] node {$13$} (t1);
\path[->] (t2) edge[loop above, inpolicy] node {$\mathbf{14}$} (t2);
\path[->] (t3) edge[loop above] node {$15$} (t3);

\path[->] (t2) edge[pos=0.75, above, gray] node { } (t1);
\path[->] (t3) edge[pos=0.75, above, gray] node { } (t2);
\path[->] (t3) edge[bend right=18, pos=0.75, above, gray] node { } (t1);

\path[->] (t1) edge[bend right=\bend, pos=\weightpos, inpolicy] node[left] {$\mathbf{16}$} (b1);
\path[->] (t2) edge[bend right=\bend, pos=\weightpos] node[left] {$16$} (b2);
\path[->] (t2) edge[bend right=\bend, pos=\weightpos] node[left] {$16$} (b1);
\path[->] (t3) edge[bend right=\bend, pos=0.72, inpolicy] node[left] {$\mathbf{16}$} (b1);
\path[->] (t3) edge[bend right=\bend, pos=\weightpos] node[left] {$16$} (b2);
\path[->] (t3) edge[bend right=\bend, pos=\weightpos] node[left] {$16$} (b3);

\path[->] (b1) edge[bend right=\bend,pos=\weightpos, gray] node[right] {\zerolabel} (t1);
\path[->] (b1) edge[bend right=\bend,pos=\weightpos, inpolicy] node[right] {$\mathbf{0}$} (t2);
\path[->] (b1) edge[bend right=\bend,pos=\weightpos, gray] node[right] {\zerolabel} (t3);
\path[->] (b2) edge[bend right=\bend,pos=\weightpos, gray] node[right] {\zerolabel} (t1);
\path[->] (b2) edge[bend right=\bend,pos=\weightpos, gray] node[right] {\zerolabel} (t2);
\path[->] (b2) edge[bend right=\bend,pos=\weightpos, gray] node[right] {\zerolabel} (t3);
\path[->] (b3) edge[bend right=\bendsingle,pos=\weightpos, gray] node[right] {\zerolabel} (t1);
\path[->] (b3) edge[bend right=\bendsingle,pos=\weightpos, gray] node[right] {\zerolabel} (t2);
\path[->] (b3) edge[bend right=\bend,pos=\weightpos, gray] node[right] {\zerolabel} (t3);

\path[->] (b2) edge[pos=0.6, below, inpolicy] node {$\mathbf{16}$} (b1);
\path[->] (b3) edge[pos=0.6, below] node {$16$} (b2);
\path[->] (b3) edge[bend left = 18, pos=0.6, below, inpolicy] node {$\mathbf{16}$} (b1);

\end{tikzpicture}%
        }
        \caption{Policy $\policy_{2,1}$}
    \end{subfigure}
    \hfill
    \begin{subfigure}{0.4\textwidth}
        \centering
        \resizebox{0.8\linewidth}{!}{%
            \begin{tikzpicture}[->,>=stealth,shorten >=1pt,auto,node distance=2.5cm,scale=0.7,transform shape]

\tikzstyle{state}=[circle,draw,minimum size=1.1cm]
\def\weightpos{0.8};
\def\bend{5};
\def\bendsingle{0};
\def\zerolabel{\,};

\node[state] (t1) at (4,0) {$t_1$};
\node[state] (t2) at (8,0) {$t_2$};
\node[state] (t3) at (12,0) {$t_3$};

\node[state] (b1) at (4,-4) {$b_1$};
\node[state] (b2) at (8,-4) {$b_2$};
\node[state] (b3) at (12,-4) {$b_3$};

\path[->] (t1) edge[loop above] node {$13$} (t1);
\path[->] (t2) edge[loop above, inpolicy] node {$\mathbf{14}$} (t2);
\path[->] (t3) edge[loop above] node {$15$} (t3);

\path[->] (t2) edge[pos=0.75, above, gray] node { } (t1);
\path[->] (t3) edge[pos=0.75, above, gray] node { } (t2);
\path[->] (t3) edge[bend right=18, pos=0.75, above, gray] node { } (t1);

\path[->] (t1) edge[bend right=\bend, pos=\weightpos, inpolicy] node[left] {$\mathbf{16}$} (b1);
\path[->] (t2) edge[bend right=\bend, pos=\weightpos] node[left] {$16$} (b2);
\path[->] (t2) edge[bend right=\bend, pos=\weightpos] node[left] {$16$} (b1);
\path[->] (t3) edge[bend right=\bend, pos=0.75] node[left] {$16$} (b1);
\path[->] (t3) edge[bend right=\bend, pos=\weightpos] node[left] {$16$} (b2);
\path[->] (t3) edge[bend right=\bend, pos=\weightpos, inpolicy] node[left] {$\mathbf{16}$} (b3);

\path[->] (b1) edge[bend right=\bend,pos=\weightpos, gray] node[right] {\zerolabel} (t1);
\path[->] (b1) edge[bend right=\bend,pos=\weightpos, inpolicy] node[right] {$\mathbf{0}$} (t2);
\path[->] (b1) edge[bend right=\bend,pos=\weightpos, gray] node[right] {\zerolabel} (t3);
\path[->] (b2) edge[bend right=\bend,pos=\weightpos, gray] node[right] {\zerolabel} (t1);
\path[->] (b2) edge[bend right=\bend,pos=\weightpos, gray] node[right] {\zerolabel} (t2);
\path[->] (b2) edge[bend right=\bend,pos=\weightpos, gray] node[right] {\zerolabel} (t3);
\path[->] (b3) edge[bend right=\bendsingle,pos=\weightpos, gray] node[right] {\zerolabel} (t1);
\path[->] (b3) edge[bend right=\bendsingle,pos=\weightpos, gray] node[right] {\zerolabel} (t2);
\path[->] (b3) edge[bend right=\bend,pos=\weightpos, gray] node[right] {\zerolabel} (t3);

\path[->] (b2) edge[pos=0.6, below, inpolicy] node {$\mathbf{16}$} (b1);
\path[->] (b3) edge[pos=0.6, below, inpolicy] node {$\mathbf{16}$} (b2);
\path[->] (b3) edge[bend left = 18, pos=0.6, below] node {$16$} (b1);

\end{tikzpicture}%
        }
        \caption{Policy $\policy_{2,3}$}
    \end{subfigure}
    \hfill
    \begin{subfigure}{0.4\textwidth}
        \centering
        \resizebox{0.8\linewidth}{!}{%
            \begin{tikzpicture}[->,>=stealth,shorten >=1pt,auto,node distance=2.5cm,scale=0.7,transform shape]

\tikzstyle{state}=[circle,draw,minimum size=1.1cm]
\def\weightpos{0.8};
\def\bend{5};
\def\bendsingle{0};
\def\zerolabel{\,};

\node[state] (t1) at (4,0) {$t_1$};
\node[state] (t2) at (8,0) {$t_2$};
\node[state] (t3) at (12,0) {$t_3$};

\node[state] (b1) at (4,-4) {$b_1$};
\node[state] (b2) at (8,-4) {$b_2$};
\node[state] (b3) at (12,-4) {$b_3$};

\path[->] (t1) edge[loop above] node {$13$} (t1);
\path[->] (t2) edge[loop above, inpolicy] node {$\mathbf{14}$} (t2);
\path[->] (t3) edge[loop above, inpolicy] node {$\mathbf{15}$} (t3);

\path[->] (t2) edge[pos=0.75, above, gray] node { } (t1);
\path[->] (t3) edge[pos=0.75, above, gray] node { } (t2);
\path[->] (t3) edge[bend right=18, pos=0.75, above, gray] node { } (t1);

\path[->] (t1) edge[bend right=\bend, pos=\weightpos, inpolicy] node[left] {$\mathbf{16}$} (b1);
\path[->] (t2) edge[bend right=\bend, pos=\weightpos] node[left] {$16$} (b2);
\path[->] (t2) edge[bend right=\bend, pos=\weightpos] node[left] {$16$} (b1);
\path[->] (t3) edge[bend right=\bend, pos=0.75] node[left] {$16$} (b1);
\path[->] (t3) edge[bend right=\bend, pos=\weightpos] node[left] {$16$} (b2);
\path[->] (t3) edge[bend right=\bend, pos=\weightpos] node[left] {$16$} (b3);

\path[->] (b1) edge[bend right=\bend,pos=\weightpos, gray] node[right] {\zerolabel} (t1);
\path[->] (b1) edge[bend right=\bend,pos=\weightpos, inpolicy] node[right] {$\mathbf{0}$} (t2);
\path[->] (b1) edge[bend right=\bend,pos=\weightpos, gray] node[right] {\zerolabel} (t3);
\path[->] (b2) edge[bend right=\bend,pos=\weightpos, gray] node[right] {\zerolabel} (t1);
\path[->] (b2) edge[bend right=\bend,pos=\weightpos, gray] node[right] {\zerolabel} (t2);
\path[->] (b2) edge[bend right=\bend,pos=\weightpos, gray] node[right] {\zerolabel} (t3);
\path[->] (b3) edge[bend right=\bendsingle,pos=\weightpos, gray] node[right] {\zerolabel} (t1);
\path[->] (b3) edge[bend right=\bendsingle,pos=\weightpos, gray] node[right] {\zerolabel} (t2);
\path[->] (b3) edge[bend right=\bend,pos=\weightpos, gray] node[right] {\zerolabel} (t3);

\path[->] (b2) edge[pos=0.6, below, inpolicy] node {$\mathbf{16}$} (b1);
\path[->] (b3) edge[pos=0.6, below, inpolicy] node {$\mathbf{16}$} (b2);
\path[->] (b3) edge[bend left = 18, pos=0.6, below] node {$16$} (b1);

\end{tikzpicture}%
        }
        \caption{Policy $\finalpol_2$}
    \end{subfigure}
    \caption{Examples of policies. Thick lines correspond to policy choices. Unlabeled (gray) edges have weight 0.}
    \label{fig:policy-families}
\end{figure*}
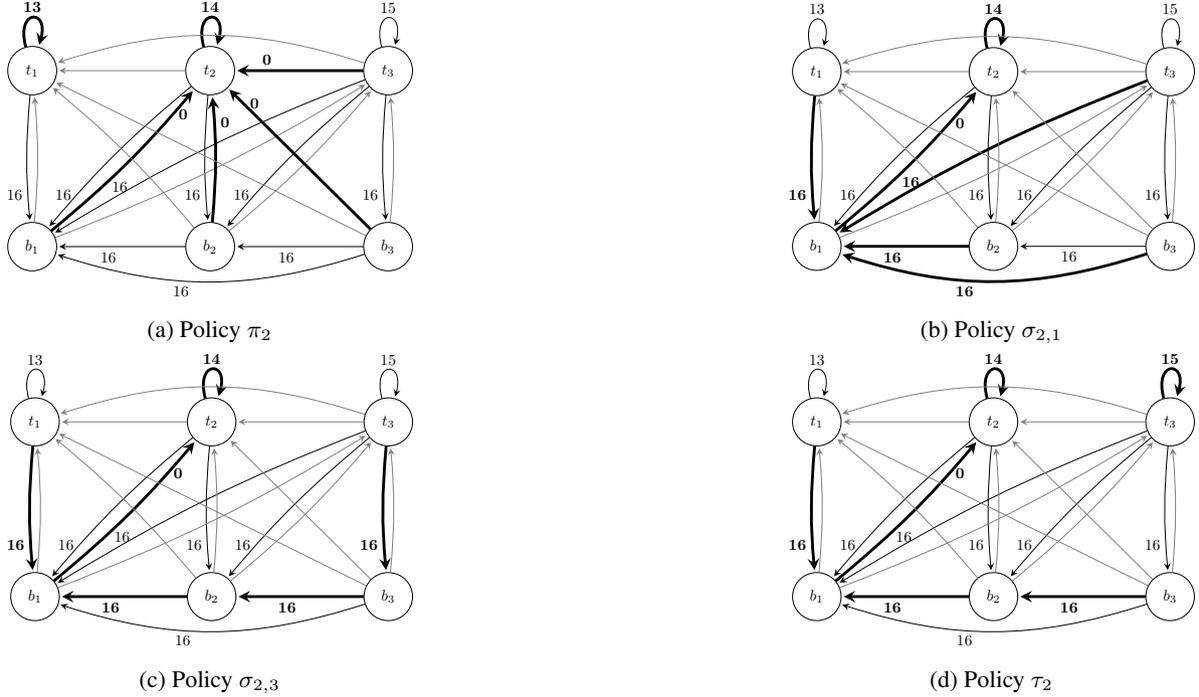

\subsection{Howard's Algorithm on $P_n$}
Given the DMDP $P_n$, we show that the sequence of policies that appear in Howard's algorithm is as follows. 
\begin{align}
    &\initpol_1 \to \policy_{1, 1} \to \policy_{1, 2} \to \finalpol_{1} \notag\\
    &\to \initpol_2 \to \policy_{2, 1} \to \policy_{2, 2} \to \policy_{2, 3} \to \finalpol_{2}\notag\\
    &\mathrel{\vdots}\label{eq:howard-seq}\\
    &\to \initpol_{n-1} \to \policy_{n-1, 1} \to \ldots \to \policy_{n-1,n} \to \finalpol_{n-1}\notag\\
    &\to \initpol_{n} \to \policy_{n, 1} \to \ldots \to \policy_{n, n-1}\notag
\end{align}

\paragraph{Intuition.} The $n$ highest mean-payoff cycles in the graph are the self-loops of $t_1$ through $t_n$. At a high level, the algorithm "finds" those cycles one after another, taking roughly $i$ iterations to find the cycle at $t_i$ after finding the cycle at $t_{i-1}$. This happens because whenever the next-best cycle is found, all "progress" the algorithm made so far in the rest of the graph is lost.
In policies $\initpol_{i}$, the lasso-shaped runs of all vertices (except $t_{i-1}$) end up in the cycle at $t_i$. Now, in the $(\policy_{i,j})_j$ chain of policies, the vertices keep adding an edge of weight $(n+1)^2$ to the path of their run by including more of the $b$-vertices (the "deceleration lane"). Since the weight of the deceleration lane edges is greater than the weight of the self-loops, the $t$-vertices only "realize" they can do better than their current run by using their self-loop when they can no longer improve their path using the deceleration lane. However, by the tie-breaking rule, the vertices only ever add one additional vertex of the deceleration lane to their path, so it takes all $i+1$ iterations from the $(\policy_{i,j})_j$ chain until the best improvement by appraisal for vertex $t_{i+1}$ is to use its self-loop, which happens in the iteration to $\finalpol_i$ (vertices $t_{i+2},...,t_n$ can still do more deceleration lane improvements so don't use their self-loop yet). In the next iteration, all vertices in the deceleration lane as well as $t_{i+2},...,t_n$ "realize" that in their current run they do not end up in the highest mean-payoff cycle, since a new, better cycle formed at $t_{i+1}$. Thus, instead of doing the next improvement in the deceleration lane, they all choose their edge directly to the now-best cycle at $t_{i+1}$. Thus, all "progress" in the deceleration lane is lost.
This continues until the cycle at $t_n$ is formed, when the algorithm does a final run through the deceleration lane before halting in the optimal policy.

\paragraph{Formal Argument.} In the following, we formally show that given a policy in the sequence, the Bellman operator returns the next policy in the sequence. The initial policy is $\initpol_1$, because $t_1$ is the vertex with lowest index in the ordering and all vertices have outgoing edges to $t_1$. 

\begin{lemma}
\label{lem:howard-sequence}
    The following assertions hold:
    \begin{compactenum}
        \item\label{item:init-to-pol} for $1 \le i \le n$, it holds that $\B(\initpol_{i}) = \policy_{i, 1}$; 
        \item\label{item:pol-to-pol} for \(1 \le i \le n\) and \(1 \le j \le \begin{cases} 
        i & 1 \le i \le n-1 \\ 
        n-2 & i = n 
    \end{cases}\), it holds that \(\mathcal{B}(\policy_{i,j}) = \policy_{i,j+1}\);
        \item\label{item:pol-to-final}  for $1 \le i \le n-1$, it holds that $\B(\policy_{i,i+1}) = \finalpol_{i}$;
         \item\label{item:final-to-init}  for $1 \le i \le n-1$, it holds that $\B(\finalpol_{i}) = \initpol_{i+1}$; and
        
        \item\label{item:pol-to-end} $\B(\policy_{n,n-1}) = \policy_{n,n-1}$. 
    \end{compactenum}
\end{lemma}

\begin{proof}
    We formally prove the first two items of the result. We abbreviate the proofs of the rest, where details can be filled in similarly.
    For notational simplicity, we denote by \[\vali{i} = w_n(t_i,t_i)= n(n+1)+i\] the weight of the self-looping edge of vertex $t_i$ and by \[\wmax=(n+1)^2\] the weight of all other non-zero edges. 
    The following inequalities will be essential throughout the proof:
    \begin{compactitem}
        \item it holds that 
        \begin{equation}
        \label{eq:pol-to-pol1} 
        \wmax > \vali{n} > \dots > \vali{i}>\dots>\vali{1}>0\,;
        \end{equation}
        \item
        for all $0\le k\leq n$ and $1\leq i\leq n$, it holds that 
         \begin{equation}
            \label{eq:pol-to-pol2} 
            k\wmax < (k+1)\vali{i}\,;  
         \end{equation}
        \item for all $k\leq n+1$ and $1\leq i\leq n$, it holds that 
        \begin{equation}
            \label{eq:pol-to-pol3} 
            (k-1)\wmax - k\vali{i} < k\wmax - (k+1)\vali{i}\,.
        \end{equation}
    \end{compactitem}

    \paragraph{Proof of \Cref{item:init-to-pol}.}
     We first give a quick intuitive rationale: In policy $\initpol_{i}$, all vertices that have an edge to directly move to $t_i$, the highest-value cycle of value $\vali{i}$, use it. $t_{i-1}$ loops to itself and the vertices $t_k, k\leq i-2$, that have no edge directly to $t_i$, use the edge to $b_{k}$ (which uses the edge to $t_i$). Since $t_i$ is the highest-value cycle, all vertices will pick their next edge so that they end up there, and between their options choose to maximize the weight of the path to $t_i$, minus $\vali{i}$ times the path length. All the vertices that currently use the edge directly to $t_i$, and $t_{i-1}$, can improve this quantity of the path to $t_i$ by using one of their edges to any of the $b$-vertices, thus adding an edge of weight $\wmax>\vali{i}$ to their path. By the tie breaking rule they pick the edge to $b_1$. The remaining vertices, $t_k$, for $k\leq i-2$, in this iteration cannot increase the average weight of their path to $t_i$, so do not change their used edge.

      We now formalize this intuition. First, we compute the mean-payoff for the policy $\initpol_i$. For all vertices $\vertex \in \Vertices \setminus \{ t_{i-1} \}$, the cycle of the lasso-shaped play $P^{\initpol_i}_\vertex$ is of form $\langle
    t_i \rangle$. Since $\initpol_i(t_{i-1}) = t_{i-1}$, therefore the cycle of $P^{\initpol_i}_{t_{i-1}}$ is of form $\langle t_{i-1} \rangle$. Therefore, the mean-payoff for the policy $\initpol_i$ is 
    \begin{equation}
    \label{eq:mp-initpol}
        \val^{\initpol_i}(\vertex) = \begin{cases}
            \vali{i-1} & \vertex = t_{i-1}\\
            \vali{i}  & \text{otherwise}\,.
        \end{cases}
    \end{equation}
    We now compute the potential function $\pot^{\initpol_i}$. Fo vertices $t_{i-1}$ and $t_i$, the potential is 0, because the path of their lasso-shaped play is of length 0. For $1 \le k \le i-2$, the path of the play $P^{\initpol_i}_{t_k}$ is of form $\langle t_k, b_k, t_i \rangle$. Therefore, we have
    \begin{align*}
        \pot^{\initpol_i}(t_k) 
        &= (\wmax - \vali{i})+ (0 - \vali{i} )\\
        &= \wmax - 2\vali{i} \,.
    \end{align*}
    For $1 \le k \le n$, the path of the play $P^{\initpol_i}_{b_k}$ is of form $\langle b_k, t_i \rangle$. Since the weight of the edge $(b_k, t_i)$ is $0$, the potential of the vertex $b_k$ is $\pot^{\initpol_i}(b_k) = - \vali{i}$.
    For $i < k$, the path of the play $P^{\initpol_i}_{t_k}$ is of form $\langle t_k, t_i \rangle$. Since the weight of the edge $(t_k, t_i)$ is $0$, the potential of the vertex $t_k$ is $\pot^{\initpol_i}(t_k) = - \vali{i}$.
    To consolidate the previous computations, the potential function for the policy $\initpol_i$ is
    \begin{equation}
    \label{eq:pot-initpol}    
        \pot^{\initpol_i}(\vertex) = \begin{cases}
            0& \vertex = t_k \quad \forall k \in \{i-1, i\}\\
            \wmax - 2\vali{i} & \vertex = t_k \quad \forall  k \le i-2\\
            - \vali{i}& \vertex = b_k \quad \forall k \le n\\
            - \vali{i}& \vertex = t_k \quad \forall  k>i
        \end{cases}
    \end{equation}
    We are now ready to show that \(\B(\initpol_i)(\vertex) = \policy_{i,1}(\vertex)\) for all \(\vertex \in \Vertices_n\). Recall that the appraisal for edge \((\vertex, \otherver)\) under \(\initpol_i\) is:
    \[
        \appr^{\initpol_i}(\vertex, \otherver) = \left(\val^{\initpol_i}(\otherver), \, \weight_n(\vertex, \otherver) - \val^{\initpol_i}(\otherver) + \pot^{\initpol_i}(\otherver)\right),
    \]
    compared lexicographically with tie-breaking as described in \Cref{Section: Preliminaries}. 
    By \Cref{eq:mp-initpol} we know that the first element of the appraisal is maximized by all $\otherver\in \Edges(\vertex)\backslash\{t_{i-1}\}$. Thus, the second element of the appraisal, $\apprtwo^{\initpol_i}(\vertex, \otherver)$, is deciding. We now precede by, unfortunately quite tediously but inevitability, calculating $\apprtwo^{\initpol_i}(\vertex, \otherver)$ for all $\vertex\in\Vertices$ and $\otherver\in\Edges(\vertex)\backslash\{t_{i-1}\}$.

    For all $j \le i-2$, $\vertex = t_j$, we get that 
    \[\apprtwo^{\initpol_i}(\vertex, \otherver)
    =\begin{cases} 
        \wmax - 3\vali{i} & \otherver=t_k, k\le j-1 \\
        \wmax + \vali{j} - 3\vali{i} & \otherver=t_j \\
        \wmax - 2\vali{i} & \otherver=b_k, k\le j
      \end{cases}\,.\] This is maximized by $\otherver=b_k$ for any $k\le j$. By the tie-breaking rule $t_j$ favors $\initpol_i(t_j)=b_j$ if possible, so that $\B(\initpol_{i})(t_j)=b_j=\policy_{i,1}(t_j).$

       For $\vertex = t_{i-1}$, we get that 
    \[\apprtwo^{\initpol_i}(\vertex, \otherver)
    =\begin{cases} 
              \wmax - 3\vali{i} & \otherver=t_k, k\le i-2 \\
              \wmax - 2\vali{i} & \otherver=b_k, k\le i-1 \\
      \end{cases}\,.\] This is maximized by $\otherver=b_k$ for any $k\le i-1$. Since $\initpol_i(t_{i-1})=t_{i-1}$,  the tie-breaking rule will favor based on the vertex ordering, so that $\B(\initpol_{i})(t_{i-1})=b_1=\policy_{i,1}(t_{i-1}).$

       For $\vertex = t_{i}$, we get that 
    \[\apprtwo^{\initpol_i}(\vertex, \otherver)
    =\begin{cases} 
            \wmax - 3\vali{i} & \otherver=t_k, k\le i-2 \\
            0 &\otherver=t_i \\ 
            \wmax - 2\vali{i} &\otherver=b_k, k\le i  \\
      \end{cases}\,.\] This is maximized by $\otherver=t_i$. Thus, it holds that $\B(\initpol_{i})(t_{i})=t_i=\policy_{i,1}(t_i).$

     For all $j > i$, $\vertex = t_j$, we get that 
    \[\apprtwo^{\initpol_i}(\vertex, \otherver)
    =\begin{cases} 
             \wmax - 3\vali{i}& \otherver=t_k, k\le i-2 \\
              -\vali{i}& \otherver=t_i  \\
              -2\vali{i}& \otherver=t_k, i<k<j  \\
              \vali{j}- 2\vali{i}& \otherver=t_j  \\
              \wmax - 2\vali{i}& \otherver=b_k, k\le j  \\
      \end{cases}\,.\] This is maximized by $\otherver=b_k$ for any $k\le j$. Since $\initpol_i(t_{j})=t_{i}$, the tie-breaking rule will favor based on the vertex ordering, so that $\B(\initpol_{i})(t_{j})=b_1=\policy_{i,1}(t_j).$

       For all $1\le j \le n$, $\vertex = b_j$, we get that 
      \[\apprtwo^{\initpol_i}(\vertex, \otherver)
    =\begin{cases} 
             \wmax - 3\vali{i}& \otherver=t_k, k\le i-2 \\
             -\vali{i}& \otherver=t_i \\
             -2\vali{i}& \otherver=t_k, k>i \\
             \wmax - 2\vali{i}& \otherver=b_k, k< j \\
      \end{cases}\,.\] For $j\ge 2$, this is maximized by $\otherver=b_k$ for any $k< j$. Since $\initpol_i(t_{j})=t_{i}$, the tie-breaking rule will favor based on the vertex ordering, so that $\B(\initpol_{i})(b_{j})=b_1=\policy_{i,1}(b_j).$ For $j=1$, it is maximized for $u=t_i$, so $\B(\initpol_{i})(b_1)=t_i=\policy_{i,1}(b_1).$
      Thus, we conclude that $\B(\initpol_{i}) = \policy_{i,1}$.

    \paragraph{Proof of \Cref{item:pol-to-pol}.}
    We first give a quick intuitive rationale: In $\policy_{i,j}$, all paths end up in the same cycle, the self-loop at $t_i$ with value $\vali{i}$. Thus, in the iteration to $\policy_{i,j+1}$, the other vertices choose their next edge solely to maximize the weight of their path to $t_i$, minus $\vali{i}$ times the path length. The optimal path for that is using as much of the deceleration lane $\langle b_n,...,b_1\rangle$ as possible, since its edges have weight $\wmax > \vali{i}$. $b_1,...,b_{j+1}$ already move to their predecessor, so for all vertices that can (except $t_i$), it is optimal to pick the edge to to $b_{j+1}$ (which they do, by the tie-breaking rule) to maximize the edges in the deceleration lane in their path. Thus, in $\policy_{i,j+1}$ the vertices that can, all pick the edge to $b_{j+1}$ (except $t_i$), while the other vertices cannot further improve their path and hence pick the same edge as in $\policy_{i,j}$.
    
    Formally, we consider policy $\policy_{i,j}$ for any $1\leq i \le n$ and $1 \le j \le \begin{cases} 
        i & 1 \le i \le n-1 \\ 
        n-2 & i = n 
    \end{cases}$. For all vertices $\vertex\in\Vertices$, the cycle of the lasso-shaped play $P^{\policy_{i,j}}$ is $\langle t_i \rangle$. Therefore, for all vertices $\vertex\in\Vertices$, it holds that $\val^{\policy_{i,j}}(\vertex) = \vali{i}$. Hence, $\B(\policy_{i,j})(\vertex)$ solely depends on $\apprtwo^{\policy_{i,j}}(\vertex,\otherver).$
    We  precede by calculating $\apprtwo^{\policy_{i,j}}(\vertex,\otherver).$ for all $\vertex\in\Vertices$ and $\otherver\in\Edges(\vertex)$. 
    We will first assume $j>1$, since $\policy^{i,j}$ is slightly different at $j=1$. We will treat this special case at the end.
    
    For $\vertex = t_{i}$, we get that \begin{multline*}
    \apprtwo^{\policy_{i,j}}(\vertex, \otherver)
    \\ =\begin{cases} 
            k\wmax  - (k+2)\vali{i} & \otherver=t_k, k\le j \\
            j\wmax  - (j+2)\vali{i} & \otherver=t_k, j<k<i \\
            0 &\otherver=t_i \\ 
            k\wmax  - (k+1)\vali{i}&\otherver=b_k, k\le j  \\
            (j+1)\wmax  - (j+2)\vali{i}&\otherver=b_k, j<k\le i 
      \end{cases}\,.\end{multline*} By \Cref{eq:pol-to-pol1,eq:pol-to-pol2,eq:pol-to-pol3}, we know that this is maximized by $\otherver=t_i$ so that $\B(\policy_{i,j})(t_{i})=t_i = \policy_{i,j+1}(t_i).$

      For $j<h<i$, $\vertex = t_{h}$, we get that  \begin{multline*}\apprtwo^{\policy_{i,j}}(\vertex, \otherver)
    \\ =\begin{cases} 
            k\wmax  - (k+2)\vali{i} & \otherver=t_k, k\le j \\
            j\wmax  - (j+2)\vali{i} & \otherver=t_k, j<k<h \\
            j\wmax + \vali{h} - (j+2)\vali{i} & \otherver=t_h \\
            k\wmax  - (k+1)\vali{i}&\otherver=b_k, k\le j  \\
            (j+1)\wmax  - (j+2)\vali{i}&\otherver=b_k, j<k\le h 
      \end{cases}\,.\end{multline*} By \Cref{eq:pol-to-pol1,eq:pol-to-pol2,eq:pol-to-pol3}, this is maximized by $\otherver=b_k$ for all $j<k\leq h$. Since $\policy_{i,j}(v)=b_j$, tie-breaking is done by vertex order so that $\B(\policy_{i,j})(t_{h})=b_{j+1} = \policy_{i,j+1}(t_h).$  

      For $h\leq j, h<i$, $\vertex = t_{h}$, we get that  \begin{multline*}\apprtwo^{\policy_{i,j}}(\vertex, \otherver)
    \\ =\begin{cases} 
            k\wmax  - (k+2)\vali{i} & \otherver=t_k, k< h \\
            h\wmax + \vali{h} - (h+2)\vali{i} & \otherver=t_h \\
            k\wmax  - (k+1)\vali{i}&\otherver=b_k, k\le h  \\
      \end{cases}\,.\end{multline*} By \Cref{eq:pol-to-pol1,eq:pol-to-pol2,eq:pol-to-pol3}, this is maximized by $\otherver=b_h$ so that $\B(\policy_{i,j})(t_{h})=b_{h} = \policy_{t_h}.$  

    For $i< h$, $\vertex = t_{h}$, we get that  \begin{multline*}\apprtwo^{\policy_{i,j}}(\vertex, \otherver)
    \\ =\begin{cases} 
            k\wmax  - (k+2)\vali{i} & \otherver=t_k, k\le j \\
            j\wmax  - (j+2)\vali{i} & \otherver=t_k, j<k<h, k\neq i \\
            -\vali{i} & \otherver=t_i \\
            j\wmax + \vali{h} - (j+2)\vali{i} & \otherver=t_h \\
            k\wmax  - (k+1)\vali{i}&\otherver=b_k, k\le j  \\
            (j+1)\wmax  - (j+2)\vali{i}&\otherver=b_k, j<k\le h 
      \end{cases}\,.\end{multline*} By \Cref{eq:pol-to-pol1,eq:pol-to-pol2,eq:pol-to-pol3}, this is maximized by $\otherver=b_k$ for all $j<k\leq h$. Since $\policy_{i,j}(v)=b_j$, tie-breaking is done by vertex order so that $\B(\policy_{i,j})(t_{h})=b_{j+1}=\policy_{i,j+1}(t_h).$ 

        For $1\leq h \leq n$, $\vertex = b_{h}$, we get that  \begin{multline*}\apprtwo^{\policy_{i,j}}(\vertex, \otherver)
    \\ =\begin{cases} 
            k\wmax  - (k+2)\vali{i} & \otherver=t_k, k\le j \\
            j\wmax  - (j+2)\vali{i} & \otherver=t_k, j<k, k\neq i \\
            -\vali{i} & \otherver=t_i \\
            k\wmax  - (k+1)\vali{i}&\otherver=b_k, k\le j, k<h  \\
            (j+1)\wmax  - (j+2)\vali{i}&\otherver=b_k, j<k<h 
      \end{cases}\,.\end{multline*} By \Cref{eq:pol-to-pol1,eq:pol-to-pol2,eq:pol-to-pol3}, for $j+1\ge h$, this is maximized by $\otherver=b_{h-1}$, so that $\B(\policy_{i,j})(b_{h})=b_{h-1}=\policy_{i,j+1}(b_h).$ For $j+1<h$, this is maximized by $b_{k}$ for all $j<k<h$. Since $\policy_{i,j}(b_h)=b_j$, tie-breaking is done by vertex order so that $\B(\policy_{i,j})(b_{h})=b_{j+1}=\policy_{i,j+1}(b_h).$ 

       Consider $j=1$. $\apprtwo^{\policy_{i,j}}(\vertex, \otherver)$ generally is unchanged from the case $j>1$, except for $\otherver=t_k$ for $1<k\le i-2$. In these cases, due to the one extra edge in the path of run $P^{\policy_{i,j}}_\vertex$,  $\apprtwo^{\policy_{i,j}}(\vertex, \otherver)$ is greater by $\wmax - \vali{i}$. This does not change which $\otherver$'s maximize $\apprtwo^{\policy_{i,j}}(\vertex, \otherver)$, so the results for $\B(\policy_{i,j})$ outlined above for $j>1$ hold for $j=1$ as well.
       
       We conclude that $\B(\policy_{i,j})=\policy_{i,j+1}$.

    \paragraph{Proof of \Cref{item:pol-to-final}.}
    This case is highly similar to \Cref{item:pol-to-pol}. In the iteration on $\policy_{i,i+1}$, all vertices except $t_{i+1}$ follow the same pattern as for the iteration from $\policy_{i,j}$ to $\policy_{i,j+1}$ for $j\leq i$. Vertex $t_{i+1}$, though, can no longer use an edge to a higher $b_k$ to add another edge of weight $\wmax>\vali{i}$ to its path to $t_i$. However, (differently to the vertices $t_k$ for $k < i$), the edge to itself has a value higher than the value of its current run's cycle $\langle t_i \rangle$, i.e. $\vali{i+1}>\vali{i}$, so it will use this edge to add it to its path. Thus, in the iteration from $\policy_{i,i+1}$ to $\finalpol_i$, all vertices will behave as in \Cref{item:pol-to-pol}, except for $t_{i+1}$, which picks its self-loop.
    Formally, for all $\vertex,\otherver \in \Vertices$, $\appr^{\policy_{i,j}}(\vertex, \otherver)$ for $j=i+1$ follows the same pattern as for $j\leq i$. We can check that for all vertices $\vertex \neq t_{i+1}$, also the pattern of which $\otherver$'s maximize it stays unchanged, so that $\B(\policy_{i,i+1})(\vertex) = {\policy_{i,i+2}}(\vertex) = {\finalpol_{i}}(\vertex)$ \footnote{Note that  policy $\sigma_{i,i+2}$ does not appear in the algorithm, but we use it here to illustrate this point.} for $\vertex \neq t_{i+1}$. 
    For $\vertex = t_{i+1}$, by \Cref{eq:pol-to-pol1,eq:pol-to-pol2,eq:pol-to-pol3} $\apprtwo^{\policy_{i,i+1}}(t_{i+1}, \otherver)$ (and thus $\appr^{\policy_{i,i+1}}(t_{i+1}, \otherver)$) is maximized by $\otherver=t_{i+1}$ so that $\B(\policy_{i,i+1})(t_{i+1}) = t_{i+1} = {\finalpol_{i}(t_{i+1})}.$
    
    \paragraph{Proof of \Cref{item:final-to-init}.}
    Intuitively, in $\finalpol_{i}$, the best cycle is $\langle t_{i+1} \rangle$ but currently no vertex other than $t_{i+1}$ ends up there. Thus, $t_{i+1}$ has the uniquely highest value, so all vertices with edges to $t_{i+1}$ will pick them. For vertices that do not have an edge to $t_{i+1}$, those are $t_k$ for $k\leq i+1$, this iteration step still is identical to the iteration from $\policy_{i,i}$ to $\policy_{i,i+1}$ in which they did not change the edge they use.
    Formally, for all vertices $\vertex \in \Vertices \setminus \{ t_{i+1} \}$, the cycle of the lasso-shaped play $P^{\finalpol_i}_\vertex$ is of form $\langle
    t_i \rangle$. The cycle of $P^{\finalpol_i}_{t_{i+1}}$ is of form $\langle t_{i+1} \rangle$. Hence, the mean-payoff for the policy $\finalpol_i$ is 
    \begin{equation*}
        \val^{\finalpol_i}(\vertex) = \begin{cases}
            \vali{i+1} & \vertex = t_{i+1}\\
            \vali{i}  & \text{otherwise}\,.
        \end{cases}
    \end{equation*}
    Since $\val^{\finalpol_i}(\vertex)$ is uniquely maximized by $\vertex = t_{i+1}$, $\B(\finalpol_{i})(\vertex) = t_{i+1}$ for all $\vertex$ where $t_{i+1}\in\Edges(\vertex)$, i.e., for $\vertex=b_k, 1\le k\le n$ and $\vertex=t_k, i+1\le k\le n$. 
    For the remaining vertices, $\vertex=t_k, 1\le k\le i$, their appraisal function $\appr^{\finalpol_{i}}(\vertex, \otherver)$ is still the same as $\appr^{\policy_{i,i}}(\vertex, \otherver)$. Thus, $\B(\finalpol_{i})(t_k)= \B(\policy_{i,i})(t_k)=\policy_{i,i+1}(t_k)= \initpol_{i+1}(t_k)$.
    
    \paragraph{Proof of \Cref{item:pol-to-end}.}
    Again, this case is very similar to \Cref{item:pol-to-pol}. In the general case for the iteration from $\policy_{i,j}$ to $\policy_{i,j+1}$, only vertices $t_k$ with $k\geq j+1, k\neq i$ and $b_k$ with $k\geq j+2$ change the edge they use. Since we reached the end of the deceleration lane and the highest-value self-loop is being used, there are no more vertices that match these criteria, so no vertex changes the edge it uses. 
    Formally, $\appr^{\policy_{i,j}}(\vertex, \otherver)$ for $i=j+1=n$ follows the same pattern as for all values of $i,j$ considered in \Cref{item:pol-to-pol}. For all vertices $\vertex \in\Vertices$, the pattern of which $\otherver$'s maximize it also stays unchanged, so that $\B(\policy_{n,n-1}) = {\policy_{n,n}}={\policy_{n,n-1}}$.
\end{proof}

\begin{proof}[Proof of \Cref{thm:main-result}]
    The constructed DMDP has $2n$ vertices and $\frac{3n^2 + n}{2}$ edges. The absolute value of weights is $\calO(n^3)$. Therefore, the size of the DMDP is $\calO(n^2\log n)$.
    
    \Cref{lem:howard-sequence} shows that if Howard's algorithm starts with the policy $\initpol_1$, it iterates over all policies in the sequence shown in \Cref{eq:howard-seq}. Therefore, the length of the sequence is 
    \begin{equation*}
        2n + \sum_{i=1}^{n} (i+1) -3 = \frac{n^2 + 7n - 6}{2}\,,
    \end{equation*}
    where the equality follows from the sum of arithmetic series and algebraic rearrangement of terms, which yields the result.
\end{proof}

\subsection{Experimental Evaluation}
To validate our lower bound example, we implemented both Howard's policy iteration and the example in Python. The experimental evaluation confirms the sequence of policies shown by the theoretical analysis. The full implementation is publicly available at \url{https://doi.org/10.5281/zenodo.14823415}.

\section{Extensions}\label{Section:Extensions}
We discuss several extensions of \Cref{thm:main-result}. 

\paragraph{Policy Initialization.}
The number of iterations that Howard's Policy Iteration algorithm performs depends on the choice of the initial policy $\policy_0$. Two natural choices are for each vertex to use (a)~the edge to it's lowest-index neighbor or (b)~the highest-weight outgoing edge (breaking ties by vertex indices). Option (b) is the most common in literature and was used in~\citet{Howard60}. While our lower bound proof above uses option (a) to get the starting policy, it is noteworthy that our lower bound also holds for (b). In particular, we let \[
    \policy_{0}(\vertex) \defas \arg\max_{\otherver\in\Edges(\vertex)} \weight(\vertex, \otherver) =
    \begin{cases}
        t_1 & \vertex = b_1\\
        b_1 & \text{otherwise}
    \end{cases}\,.
\]
One iteration of Howard's algorithm on  $\policy_0$ leads to $\policy_{1,2}$, from which on the algorithm proceeds as described above, iterating over the policies in the sequence shown in \Cref{eq:howard-seq} starting at $\policy_{1,2}$. Compared to using $\initpol_{1}$ as an initial policy,  Howard's algorithm with initial policy $\policy_0$ only performs one iteration less. Thus, the number of iterations for this policy initialization is still $\Omega(n^2)$.

\paragraph{Discounted-sum Objectives.} In discounted-sum objectives, every edge is assigned an integer weight and the payoff is the discounted sum of these weights. Although \Cref{thm:main-result} is stated for mean-payoff objectives, it extends to discounted-sum objectives with a discount factor sufficiently close to 1 as a function of $n$, because as the discount factor approaches 1, discounting diminishes and the sum converges to the mean-payoff value. Furthermore, by Blackwell optimality~\citep{Blackwell62}, an optimal policy optimal for the mean-payoff objectives remains optimal for all discount factors sufficiently near 1.

It is an open question whether our techniques can be extended to obtain a similar lower bound for a discounted-sum objective with a constant discount factor. In this setting, a lower bound of $\Omega(n^2)$ on the number of iterations would be tight up to a factor of $\log n$ to the upper bound due to \cite{hansen2013strategy}, which applies not only to DMDPs, but also stochastic MDPs and in the 2-player setting.

\section{Conclusion and Future Work} \label{sec:conclu}
In this work, we studied Howard's policy iteration algorithm for DMDPs with mean-payoff objectives and constructed a family of examples with $2n$ vertices and $\calO(n^2)$ edges where the algorithm requires \(\Omega(n^2)\) iterations and improved the lower bound on the number of iterations to $\widetilde{\Omega}(I)$ with respect to the input size $I$. There are several interesting directions for future work. In particular, Hansen's conjecture~\citep{hansen2012worst} on the number of iterations remains a major open problem. Furthermore, the practical performance of Howard’s policy iteration, despite its high theoretical worst-case complexity, raises relevant and interesting questions. While our focus is to establish an improved theoretical lower bound for this classical algorithm, these practical concerns highlight important directions for future research. 

\begin{acknowledgements}
This research was partially supported by the ERC CoG 863818 (ForM-SMArt) grant and Austrian Science Fund (FWF) 10.55776/COE12.
\end{acknowledgements}

\bibliography{refs}
\newpage
\onecolumn

\title{Lower Bound on Howard Policy Iteration for Deterministic Markov Decision Processes\\(Supplementary Material)}
\maketitle

\appendix

\section{Sequence of Policies}
\label{sec:policy-seq}
In this section, we illustrate the sequence of policies appearing in Howard's policy iteration on the DMDP $P_3$ and for the general $P_n$. 

\subsection{The policies for our running example}
\begin{figure*}[h]
    \centering
    \begin{subfigure}{0.49\textwidth}
        \centering
        \begin{tikzpicture}[->,>=stealth,shorten >=1pt,auto,node distance=2.5cm,scale=0.7,transform shape]

\tikzstyle{state}=[circle,draw,minimum size=1.1cm]
\def\weightpos{0.8};
\def\bend{5};
\def\bendsingle{0};
\def\zerolabel{\,};

\node[state] (t1) at (4,0) {$t_1$};
\node[state] (t2) at (8,0) {$t_2$};
\node[state] (t3) at (12,0) {$t_3$};

\node[state] (b1) at (4,-4) {$b_1$};
\node[state] (b2) at (8,-4) {$b_2$};
\node[state] (b3) at (12,-4) {$b_3$};

\path[->] (t1) edge[loop above, inpolicy] node {$\mathbf{13}$} (t1);
\path[->] (t2) edge[loop above] node {$14$} (t2);
\path[->] (t3) edge[loop above] node {$15$} (t3);

\path[->] (t2) edge[pos=0.6, above, inpolicy] node {$\mathbf{0}$} (t1);
\path[->] (t3) edge[pos=0.75, above, gray] node { } (t2);
\path[->] (t3) edge[bend right=18, pos=0.75, above, inpolicy] node {$\mathbf{0}$} (t1);

\path[->] (t1) edge[bend right=\bend, pos=\weightpos] node[left] {$16$} (b1);
\path[->] (t2) edge[bend right=\bend, pos=\weightpos] node[left] {$16$} (b2);
\path[->] (t2) edge[bend right=\bend, pos=\weightpos] node[left] {$16$} (b1);
\path[->] (t3) edge[bend right=\bend, pos=0.75] node[left] {$16$} (b1);
\path[->] (t3) edge[bend right=\bend, pos=\weightpos] node[left] {$16$} (b2);
\path[->] (t3) edge[bend right=\bend, pos=\weightpos] node[left] {$16$} (b3);

\path[->] (b1) edge[bend right=\bend,pos=\weightpos, inpolicy] node[right] {$\mathbf{0}$} (t1);
\path[->] (b1) edge[bend right=\bend,pos=\weightpos, gray] node[right] {\zerolabel} (t2);
\path[->] (b1) edge[bend right=\bend,pos=\weightpos, gray] node[right] {\zerolabel} (t3);
\path[->] (b2) edge[bend right=\bendsingle,pos=\weightpos, inpolicy] node[right] {$\mathbf{0}$} (t1);
\path[->] (b2) edge[bend right=\bend,pos=\weightpos, gray] node[right] {\zerolabel} (t2);
\path[->] (b2) edge[bend right=\bend,pos=\weightpos, gray] node[right] {\zerolabel} (t3);
\path[->] (b3) edge[bend right=\bendsingle,pos=0.87,inner sep=10pt, inpolicy] node[right] {$\mathbf{0}$} (t1);
\path[->] (b3) edge[bend right=\bendsingle,pos=\weightpos, gray] node[right] {\zerolabel} (t2);
\path[->] (b3) edge[bend right=\bend,pos=\weightpos, gray] node[right] {\zerolabel} (t3);

\path[->] (b2) edge[pos=0.6, below] node {$16$} (b1);
\path[->] (b3) edge[pos=0.6, below] node {$16$} (b2);
\path[->] (b3) edge[bend left = 18, pos=0.6, below] node {$16$} (b1);

\end{tikzpicture}
        \caption{Policy $\initpol_{1}$}
    \end{subfigure}
    \hfill
    \begin{subfigure}{0.49\textwidth}
        \centering
        \begin{tikzpicture}[->,>=stealth,shorten >=1pt,auto,node distance=2.5cm,scale=0.7,transform shape]

\tikzstyle{state}=[circle,draw,minimum size=1.1cm]
\def\weightpos{0.8};
\def\bend{5};
\def\bendsingle{0};
\def\zerolabel{\,};

\node[state] (t1) at (4,0) {$t_1$};
\node[state] (t2) at (8,0) {$t_2$};
\node[state] (t3) at (12,0) {$t_3$};

\node[state] (b1) at (4,-4) {$b_1$};
\node[state] (b2) at (8,-4) {$b_2$};
\node[state] (b3) at (12,-4) {$b_3$};

\path[->] (t1) edge[loop above, inpolicy] node {$\mathbf{13}$} (t1);
\path[->] (t2) edge[loop above] node {$14$} (t2);
\path[->] (t3) edge[loop above] node {$15$} (t3);

\path[->] (t2) edge[pos=0.75, above, gray] node { } (t1);
\path[->] (t3) edge[pos=0.75, above, gray] node { } (t2);
\path[->] (t3) edge[bend right=18, pos=0.75, above, gray] node { } (t1);

\path[->] (t1) edge[bend right=\bend, pos=\weightpos] node[left] {$16$} (b1);
\path[->] (t2) edge[bend right=\bend, pos=\weightpos] node[left] {$16$} (b2);
\path[->] (t2) edge[bend right=\bend, pos=\weightpos, inpolicy] node[left] {$\mathbf{16}$} (b1);
\path[->] (t3) edge[bend right=\bend, pos=0.75, inpolicy] node[left] {$\mathbf{16}$} (b1);
\path[->] (t3) edge[bend right=\bend, pos=\weightpos] node[left] {$16$} (b2);
\path[->] (t3) edge[bend right=\bend, pos=\weightpos] node[left] {$16$} (b3);

\path[->] (b1) edge[bend right=\bend,pos=\weightpos, inpolicy] node[right] {$\mathbf{0}$} (t1);
\path[->] (b1) edge[bend right=\bend,pos=\weightpos, gray] node[right] {\zerolabel} (t2);
\path[->] (b1) edge[bend right=\bend,pos=\weightpos, gray] node[right] {\zerolabel} (t3);
\path[->] (b2) edge[bend right=\bend,pos=\weightpos, gray] node[right] {\zerolabel} (t1);
\path[->] (b2) edge[bend right=\bend,pos=\weightpos, gray] node[right] {\zerolabel} (t2);
\path[->] (b2) edge[bend right=\bend,pos=\weightpos, gray] node[right] {\zerolabel} (t3);
\path[->] (b3) edge[bend right=\bendsingle,pos=\weightpos, gray] node[right] {\zerolabel} (t1);
\path[->] (b3) edge[bend right=\bendsingle,pos=\weightpos, gray] node[right] {\zerolabel} (t2);
\path[->] (b3) edge[bend right=\bend,pos=\weightpos, gray] node[right] {\zerolabel} (t3);

\path[->] (b2) edge[pos=0.6, below, inpolicy] node {$\mathbf{16}$} (b1);
\path[->] (b3) edge[pos=0.6, below] node {$16$} (b2);
\path[->] (b3) edge[bend left = 18, pos=0.6, below, inpolicy] node {$\mathbf{16}$} (b1);

\end{tikzpicture}
        \caption{Policy $\policy_{1,1}$}
    \end{subfigure}
    \begin{subfigure}{0.49\textwidth}
        \centering
        \begin{tikzpicture}[->,>=stealth,shorten >=1pt,auto,node distance=2.5cm,scale=0.7,transform shape]

\tikzstyle{state}=[circle,draw,minimum size=1.1cm]
\def\weightpos{0.8};
\def\bend{5};
\def\bendsingle{0};
\def\zerolabel{\,};

\node[state] (t1) at (4,0) {$t_1$};
\node[state] (t2) at (8,0) {$t_2$};
\node[state] (t3) at (12,0) {$t_3$};

\node[state] (b1) at (4,-4) {$b_1$};
\node[state] (b2) at (8,-4) {$b_2$};
\node[state] (b3) at (12,-4) {$b_3$};

\path[->] (t1) edge[loop above, inpolicy] node {$\mathbf{13}$} (t1);
\path[->] (t2) edge[loop above] node {$14$} (t2);
\path[->] (t3) edge[loop above] node {$15$} (t3);

\path[->] (t2) edge[pos=0.75, above, gray] node { } (t1);
\path[->] (t3) edge[pos=0.75, above, gray] node { } (t2);
\path[->] (t3) edge[bend right=18, pos=0.75, above, gray] node { } (t1);

\path[->] (t1) edge[bend right=\bend, pos=\weightpos] node[left] {$16$} (b1);
\path[->] (t2) edge[bend right=\bend, pos=\weightpos, inpolicy] node[left] {$\mathbf{16}$} (b2);
\path[->] (t2) edge[bend right=\bend, pos=\weightpos] node[left] {$16$} (b1);
\path[->] (t3) edge[bend right=\bend, pos=0.75] node[left] {$16$} (b1);
\path[->] (t3) edge[bend right=\bend, pos=\weightpos, inpolicy] node[left] {$\mathbf{16}$} (b2);
\path[->] (t3) edge[bend right=\bend, pos=\weightpos] node[left] {$16$} (b3);

\path[->] (b1) edge[bend right=\bend,pos=\weightpos, inpolicy] node[right] {$\mathbf{0}$} (t1);
\path[->] (b1) edge[bend right=\bend,pos=\weightpos, gray] node[right] {\zerolabel} (t2);
\path[->] (b1) edge[bend right=\bend,pos=\weightpos, gray] node[right] {\zerolabel} (t3);
\path[->] (b2) edge[bend right=\bend,pos=\weightpos, gray] node[right] {\zerolabel} (t1);
\path[->] (b2) edge[bend right=\bend,pos=\weightpos, gray] node[right] {\zerolabel} (t2);
\path[->] (b2) edge[bend right=\bend,pos=\weightpos, gray] node[right] {\zerolabel} (t3);
\path[->] (b3) edge[bend right=\bendsingle,pos=\weightpos, gray] node[right] {\zerolabel} (t1);
\path[->] (b3) edge[bend right=\bendsingle,pos=\weightpos, gray] node[right] {\zerolabel} (t2);
\path[->] (b3) edge[bend right=\bend,pos=\weightpos, gray] node[right] {\zerolabel} (t3);

\path[->] (b2) edge[pos=0.6, below, inpolicy] node {$\mathbf{16}$} (b1);
\path[->] (b3) edge[pos=0.6, below, inpolicy] node {$\mathbf{16}$} (b2);
\path[->] (b3) edge[bend left = 18, pos=0.6, below] node {$16$} (b1);

\end{tikzpicture}
        \caption{Policy $\policy_{1,2}$}
    \end{subfigure}
    \hfill
    \begin{subfigure}{0.49\textwidth}
        \centering
        \begin{tikzpicture}[->,>=stealth,shorten >=1pt,auto,node distance=2.5cm,scale=0.7,transform shape]

\tikzstyle{state}=[circle,draw,minimum size=1.1cm]
\def\weightpos{0.8};
\def\bend{5};
\def\bendsingle{0};
\def\zerolabel{\,};

\node[state] (t1) at (4,0) {$t_1$};
\node[state] (t2) at (8,0) {$t_2$};
\node[state] (t3) at (12,0) {$t_3$};

\node[state] (b1) at (4,-4) {$b_1$};
\node[state] (b2) at (8,-4) {$b_2$};
\node[state] (b3) at (12,-4) {$b_3$};

\path[->] (t1) edge[loop above, inpolicy] node {$\mathbf{13}$} (t1);
\path[->] (t2) edge[loop above, inpolicy] node {$\mathbf{14}$} (t2);
\path[->] (t3) edge[loop above] node {$15$} (t3);

\path[->] (t2) edge[pos=0.75, above, gray] node { } (t1);
\path[->] (t3) edge[pos=0.75, above, gray] node { } (t2);
\path[->] (t3) edge[bend right=18, pos=0.75, above, gray] node { } (t1);

\path[->] (t1) edge[bend right=\bend, pos=\weightpos] node[left] {$16$} (b1);
\path[->] (t2) edge[bend right=\bend, pos=\weightpos] node[left] {$16$} (b2);
\path[->] (t2) edge[bend right=\bend, pos=\weightpos] node[left] {$16$} (b1);
\path[->] (t3) edge[bend right=\bend, pos=0.75] node[left] {$16$} (b1);
\path[->] (t3) edge[bend right=\bend, pos=\weightpos] node[left] {$16$} (b2);
\path[->] (t3) edge[bend right=\bend, pos=\weightpos, inpolicy] node[left] {$\mathbf{16}$} (b3);

\path[->] (b1) edge[bend right=\bend,pos=\weightpos, inpolicy] node[right] {$\mathbf{0}$} (t1);
\path[->] (b1) edge[bend right=\bend,pos=\weightpos, gray] node[right] {\zerolabel} (t2);
\path[->] (b1) edge[bend right=\bend,pos=\weightpos, gray] node[right] {\zerolabel} (t3);
\path[->] (b2) edge[bend right=\bend,pos=\weightpos, gray] node[right] {\zerolabel} (t1);
\path[->] (b2) edge[bend right=\bend,pos=\weightpos, gray] node[right] {\zerolabel} (t2);
\path[->] (b2) edge[bend right=\bend,pos=\weightpos, gray] node[right] {\zerolabel} (t3);
\path[->] (b3) edge[bend right=\bendsingle,pos=\weightpos, gray] node[right] {\zerolabel} (t1);
\path[->] (b3) edge[bend right=\bendsingle,pos=\weightpos, gray] node[right] {\zerolabel} (t2);
\path[->] (b3) edge[bend right=\bend,pos=\weightpos, gray] node[right] {\zerolabel} (t3);

\path[->] (b2) edge[pos=0.6, below, inpolicy] node {$\mathbf{16}$} (b1);
\path[->] (b3) edge[pos=0.6, below, inpolicy] node {$\mathbf{16}$} (b2);
\path[->] (b3) edge[bend left = 18, pos=0.6, below] node {$16$} (b1);

\end{tikzpicture}
        \caption{Policy $\finalpol_1$}
    \end{subfigure}
    \caption{Part I: the sequence of policies appearing in Howard's policy iteration over our running example. Thick lines correspond to policy choices. Unlabeled (gray) edges have weight 0.}
    \label{fig:policy-sequence-1}
\end{figure*}
\begin{figure*}[hp]\ContinuedFloat
    \centering
    \begin{subfigure}{0.49\textwidth}
        \centering
        \begin{tikzpicture}[->,>=stealth,shorten >=1pt,auto,node distance=2.5cm,scale=0.7,transform shape]

\tikzstyle{state}=[circle,draw,minimum size=1.1cm]
\def\weightpos{0.8};
\def\bend{5};
\def\bendsingle{0};
\def\zerolabel{\,};

\node[state] (t1) at (4,0) {$t_1$};
\node[state] (t2) at (8,0) {$t_2$};
\node[state] (t3) at (12,0) {$t_3$};

\node[state] (b1) at (4,-4) {$b_1$};
\node[state] (b2) at (8,-4) {$b_2$};
\node[state] (b3) at (12,-4) {$b_3$};

\path[->] (t1) edge[loop above, inpolicy] node {$\mathbf{13}$} (t1);
\path[->] (t2) edge[loop above, inpolicy] node {$\mathbf{14}$} (t2);
\path[->] (t3) edge[loop above] node {$15$} (t3);

\path[->] (t2) edge[pos=0.75, above, gray] node { } (t1);
\path[->] (t3) edge[pos=0.75, above, inpolicy] node {$\mathbf{0}$} (t2);
\path[->] (t3) edge[bend right=18, pos=0.75, above, gray] node { } (t1);

\path[->] (t1) edge[bend right=\bend, pos=\weightpos] node[left] {$16$} (b1);
\path[->] (t2) edge[bend right=\bend, pos=\weightpos] node[left] {$16$} (b2);
\path[->] (t2) edge[bend right=\bend, pos=\weightpos] node[left] {$16$} (b1);
\path[->] (t3) edge[bend right=\bend, pos=0.75] node[left] {$16$} (b1);
\path[->] (t3) edge[bend right=\bend, pos=\weightpos] node[left] {$16$} (b2);
\path[->] (t3) edge[bend right=\bend, pos=\weightpos] node[left] {$16$} (b3);

\path[->] (b1) edge[bend right=\bend,pos=\weightpos, gray] node[right] {\zerolabel} (t1);
\path[->] (b1) edge[bend right=\bend,pos=0.85, inpolicy] node[right] {$\mathbf{0}$} (t2);
\path[->] (b1) edge[bend right=\bend,pos=\weightpos, gray] node[right] {\zerolabel} (t3);
\path[->] (b2) edge[bend right=\bend,pos=\weightpos, gray] node[right] {\zerolabel} (t1);
\path[->] (b2) edge[bend right=\bend,pos=0.8, inpolicy] node[right] {$\mathbf{0}$} (t2);
\path[->] (b2) edge[bend right=\bend,pos=\weightpos, gray] node[right] {\zerolabel} (t3);
\path[->] (b3) edge[bend right=\bendsingle,pos=\weightpos, gray] node[right] {\zerolabel} (t1);
\path[->] (b3) edge[bend right=\bendsingle,pos=0.9, inpolicy] node[right] {$\mathbf{0}$} (t2);
\path[->] (b3) edge[bend right=\bend,pos=\weightpos, gray] node[right] {\zerolabel} (t3);

\path[->] (b2) edge[pos=0.6, below] node {$16$} (b1);
\path[->] (b3) edge[pos=0.6, below] node {$16$} (b2);
\path[->] (b3) edge[bend left = 18, pos=0.6, below] node {$16$} (b1);

\end{tikzpicture}
        \caption{Policy $\initpol_2$}
    \end{subfigure}
    \hfill
    \begin{subfigure}{0.49\textwidth}
        \centering
        \begin{tikzpicture}[->,>=stealth,shorten >=1pt,auto,node distance=2.5cm,scale=0.7,transform shape]

\tikzstyle{state}=[circle,draw,minimum size=1.1cm]
\def\weightpos{0.8};
\def\bend{5};
\def\bendsingle{0};
\def\zerolabel{\,};

\node[state] (t1) at (4,0) {$t_1$};
\node[state] (t2) at (8,0) {$t_2$};
\node[state] (t3) at (12,0) {$t_3$};

\node[state] (b1) at (4,-4) {$b_1$};
\node[state] (b2) at (8,-4) {$b_2$};
\node[state] (b3) at (12,-4) {$b_3$};

\path[->] (t1) edge[loop above] node {$13$} (t1);
\path[->] (t2) edge[loop above, inpolicy] node {$\mathbf{14}$} (t2);
\path[->] (t3) edge[loop above] node {$15$} (t3);

\path[->] (t2) edge[pos=0.75, above, gray] node { } (t1);
\path[->] (t3) edge[pos=0.75, above, gray] node { } (t2);
\path[->] (t3) edge[bend right=18, pos=0.75, above, gray] node { } (t1);

\path[->] (t1) edge[bend right=\bend, pos=\weightpos, inpolicy] node[left] {$\mathbf{16}$} (b1);
\path[->] (t2) edge[bend right=\bend, pos=\weightpos] node[left] {$16$} (b2);
\path[->] (t2) edge[bend right=\bend, pos=\weightpos] node[left] {$16$} (b1);
\path[->] (t3) edge[bend right=\bend, pos=0.72, inpolicy] node[left] {$\mathbf{16}$} (b1);
\path[->] (t3) edge[bend right=\bend, pos=\weightpos] node[left] {$16$} (b2);
\path[->] (t3) edge[bend right=\bend, pos=\weightpos] node[left] {$16$} (b3);

\path[->] (b1) edge[bend right=\bend,pos=\weightpos, gray] node[right] {\zerolabel} (t1);
\path[->] (b1) edge[bend right=\bend,pos=\weightpos, inpolicy] node[right] {$\mathbf{0}$} (t2);
\path[->] (b1) edge[bend right=\bend,pos=\weightpos, gray] node[right] {\zerolabel} (t3);
\path[->] (b2) edge[bend right=\bend,pos=\weightpos, gray] node[right] {\zerolabel} (t1);
\path[->] (b2) edge[bend right=\bend,pos=\weightpos, gray] node[right] {\zerolabel} (t2);
\path[->] (b2) edge[bend right=\bend,pos=\weightpos, gray] node[right] {\zerolabel} (t3);
\path[->] (b3) edge[bend right=\bendsingle,pos=\weightpos, gray] node[right] {\zerolabel} (t1);
\path[->] (b3) edge[bend right=\bendsingle,pos=\weightpos, gray] node[right] {\zerolabel} (t2);
\path[->] (b3) edge[bend right=\bend,pos=\weightpos, gray] node[right] {\zerolabel} (t3);

\path[->] (b2) edge[pos=0.6, below, inpolicy] node {$\mathbf{16}$} (b1);
\path[->] (b3) edge[pos=0.6, below] node {$16$} (b2);
\path[->] (b3) edge[bend left = 18, pos=0.6, below, inpolicy] node {$\mathbf{16}$} (b1);

\end{tikzpicture}
        \caption{Policy $\policy_{2,1}$}
    \end{subfigure}
    \hfill
    \begin{subfigure}{0.49\textwidth}
        \centering
        \begin{tikzpicture}[->,>=stealth,shorten >=1pt,auto,node distance=2.5cm,scale=0.7,transform shape]

\tikzstyle{state}=[circle,draw,minimum size=1.1cm]
\def\weightpos{0.8};
\def\bend{5};
\def\bendsingle{0};
\def\zerolabel{\,};

\node[state] (t1) at (4,0) {$t_1$};
\node[state] (t2) at (8,0) {$t_2$};
\node[state] (t3) at (12,0) {$t_3$};

\node[state] (b1) at (4,-4) {$b_1$};
\node[state] (b2) at (8,-4) {$b_2$};
\node[state] (b3) at (12,-4) {$b_3$};

\path[->] (t1) edge[loop above] node {$13$} (t1);
\path[->] (t2) edge[loop above, inpolicy] node {$\mathbf{14}$} (t2);
\path[->] (t3) edge[loop above] node {$15$} (t3);

\path[->] (t2) edge[pos=0.75, above, gray] node { } (t1);
\path[->] (t3) edge[pos=0.75, above, gray] node { } (t2);
\path[->] (t3) edge[bend right=18, pos=0.75, above, gray] node { } (t1);

\path[->] (t1) edge[bend right=\bend, pos=\weightpos, inpolicy] node[left] {$\mathbf{16}$} (b1);
\path[->] (t2) edge[bend right=\bend, pos=\weightpos] node[left] {$16$} (b2);
\path[->] (t2) edge[bend right=\bend, pos=\weightpos] node[left] {$16$} (b1);
\path[->] (t3) edge[bend right=\bend, pos=0.75] node[left] {$16$} (b1);
\path[->] (t3) edge[bend right=\bend, pos=\weightpos, inpolicy] node[left] {$\mathbf{16}$} (b2);
\path[->] (t3) edge[bend right=\bend, pos=\weightpos] node[left] {$16$} (b3);

\path[->] (b1) edge[bend right=\bend,pos=\weightpos, gray] node[right] {\zerolabel} (t1);
\path[->] (b1) edge[bend right=\bend,pos=\weightpos, inpolicy] node[right] {$\mathbf{0}$} (t2);
\path[->] (b1) edge[bend right=\bend,pos=\weightpos, gray] node[right] {\zerolabel} (t3);
\path[->] (b2) edge[bend right=\bend,pos=\weightpos, gray] node[right] {\zerolabel} (t1);
\path[->] (b2) edge[bend right=\bend,pos=\weightpos, gray] node[right] {\zerolabel} (t2);
\path[->] (b2) edge[bend right=\bend,pos=\weightpos, gray] node[right] {\zerolabel} (t3);
\path[->] (b3) edge[bend right=\bendsingle,pos=\weightpos, gray] node[right] {\zerolabel} (t1);
\path[->] (b3) edge[bend right=\bendsingle,pos=\weightpos, gray] node[right] {\zerolabel} (t2);
\path[->] (b3) edge[bend right=\bend,pos=\weightpos, gray] node[right] {\zerolabel} (t3);

\path[->] (b2) edge[pos=0.6, below, inpolicy] node {$\mathbf{16}$} (b1);
\path[->] (b3) edge[pos=0.6, below, inpolicy] node {$\mathbf{16}$} (b2);
\path[->] (b3) edge[bend left = 18, pos=0.6, below] node {$16$} (b1);

\end{tikzpicture}
        \caption{Policy $\policy_{2,2}$}
    \end{subfigure}
    \hfill
    \begin{subfigure}{0.49\textwidth}
        \centering
        \begin{tikzpicture}[->,>=stealth,shorten >=1pt,auto,node distance=2.5cm,scale=0.7,transform shape]

\tikzstyle{state}=[circle,draw,minimum size=1.1cm]
\def\weightpos{0.8};
\def\bend{5};
\def\bendsingle{0};
\def\zerolabel{\,};

\node[state] (t1) at (4,0) {$t_1$};
\node[state] (t2) at (8,0) {$t_2$};
\node[state] (t3) at (12,0) {$t_3$};

\node[state] (b1) at (4,-4) {$b_1$};
\node[state] (b2) at (8,-4) {$b_2$};
\node[state] (b3) at (12,-4) {$b_3$};

\path[->] (t1) edge[loop above] node {$13$} (t1);
\path[->] (t2) edge[loop above, inpolicy] node {$\mathbf{14}$} (t2);
\path[->] (t3) edge[loop above] node {$15$} (t3);

\path[->] (t2) edge[pos=0.75, above, gray] node { } (t1);
\path[->] (t3) edge[pos=0.75, above, gray] node { } (t2);
\path[->] (t3) edge[bend right=18, pos=0.75, above, gray] node { } (t1);

\path[->] (t1) edge[bend right=\bend, pos=\weightpos, inpolicy] node[left] {$\mathbf{16}$} (b1);
\path[->] (t2) edge[bend right=\bend, pos=\weightpos] node[left] {$16$} (b2);
\path[->] (t2) edge[bend right=\bend, pos=\weightpos] node[left] {$16$} (b1);
\path[->] (t3) edge[bend right=\bend, pos=0.75] node[left] {$16$} (b1);
\path[->] (t3) edge[bend right=\bend, pos=\weightpos] node[left] {$16$} (b2);
\path[->] (t3) edge[bend right=\bend, pos=\weightpos, inpolicy] node[left] {$\mathbf{16}$} (b3);

\path[->] (b1) edge[bend right=\bend,pos=\weightpos, gray] node[right] {\zerolabel} (t1);
\path[->] (b1) edge[bend right=\bend,pos=\weightpos, inpolicy] node[right] {$\mathbf{0}$} (t2);
\path[->] (b1) edge[bend right=\bend,pos=\weightpos, gray] node[right] {\zerolabel} (t3);
\path[->] (b2) edge[bend right=\bend,pos=\weightpos, gray] node[right] {\zerolabel} (t1);
\path[->] (b2) edge[bend right=\bend,pos=\weightpos, gray] node[right] {\zerolabel} (t2);
\path[->] (b2) edge[bend right=\bend,pos=\weightpos, gray] node[right] {\zerolabel} (t3);
\path[->] (b3) edge[bend right=\bendsingle,pos=\weightpos, gray] node[right] {\zerolabel} (t1);
\path[->] (b3) edge[bend right=\bendsingle,pos=\weightpos, gray] node[right] {\zerolabel} (t2);
\path[->] (b3) edge[bend right=\bend,pos=\weightpos, gray] node[right] {\zerolabel} (t3);

\path[->] (b2) edge[pos=0.6, below, inpolicy] node {$\mathbf{16}$} (b1);
\path[->] (b3) edge[pos=0.6, below, inpolicy] node {$\mathbf{16}$} (b2);
\path[->] (b3) edge[bend left = 18, pos=0.6, below] node {$16$} (b1);

\end{tikzpicture}
        \caption{Policy $\policy_{2,3}$}
    \end{subfigure}
    \hfill
    \begin{subfigure}{0.49\textwidth}
        \centering
        \begin{tikzpicture}[->,>=stealth,shorten >=1pt,auto,node distance=2.5cm,scale=0.7,transform shape]

\tikzstyle{state}=[circle,draw,minimum size=1.1cm]
\def\weightpos{0.8};
\def\bend{5};
\def\bendsingle{0};
\def\zerolabel{\,};

\node[state] (t1) at (4,0) {$t_1$};
\node[state] (t2) at (8,0) {$t_2$};
\node[state] (t3) at (12,0) {$t_3$};

\node[state] (b1) at (4,-4) {$b_1$};
\node[state] (b2) at (8,-4) {$b_2$};
\node[state] (b3) at (12,-4) {$b_3$};

\path[->] (t1) edge[loop above] node {$13$} (t1);
\path[->] (t2) edge[loop above, inpolicy] node {$\mathbf{14}$} (t2);
\path[->] (t3) edge[loop above, inpolicy] node {$\mathbf{15}$} (t3);

\path[->] (t2) edge[pos=0.75, above, gray] node { } (t1);
\path[->] (t3) edge[pos=0.75, above, gray] node { } (t2);
\path[->] (t3) edge[bend right=18, pos=0.75, above, gray] node { } (t1);

\path[->] (t1) edge[bend right=\bend, pos=\weightpos, inpolicy] node[left] {$\mathbf{16}$} (b1);
\path[->] (t2) edge[bend right=\bend, pos=\weightpos] node[left] {$16$} (b2);
\path[->] (t2) edge[bend right=\bend, pos=\weightpos] node[left] {$16$} (b1);
\path[->] (t3) edge[bend right=\bend, pos=0.75] node[left] {$16$} (b1);
\path[->] (t3) edge[bend right=\bend, pos=\weightpos] node[left] {$16$} (b2);
\path[->] (t3) edge[bend right=\bend, pos=\weightpos] node[left] {$16$} (b3);

\path[->] (b1) edge[bend right=\bend,pos=\weightpos, gray] node[right] {\zerolabel} (t1);
\path[->] (b1) edge[bend right=\bend,pos=\weightpos, inpolicy] node[right] {$\mathbf{0}$} (t2);
\path[->] (b1) edge[bend right=\bend,pos=\weightpos, gray] node[right] {\zerolabel} (t3);
\path[->] (b2) edge[bend right=\bend,pos=\weightpos, gray] node[right] {\zerolabel} (t1);
\path[->] (b2) edge[bend right=\bend,pos=\weightpos, gray] node[right] {\zerolabel} (t2);
\path[->] (b2) edge[bend right=\bend,pos=\weightpos, gray] node[right] {\zerolabel} (t3);
\path[->] (b3) edge[bend right=\bendsingle,pos=\weightpos, gray] node[right] {\zerolabel} (t1);
\path[->] (b3) edge[bend right=\bendsingle,pos=\weightpos, gray] node[right] {\zerolabel} (t2);
\path[->] (b3) edge[bend right=\bend,pos=\weightpos, gray] node[right] {\zerolabel} (t3);

\path[->] (b2) edge[pos=0.6, below, inpolicy] node {$\mathbf{16}$} (b1);
\path[->] (b3) edge[pos=0.6, below, inpolicy] node {$\mathbf{16}$} (b2);
\path[->] (b3) edge[bend left = 18, pos=0.6, below] node {$16$} (b1);

\end{tikzpicture}
        \caption{Policy $\finalpol_2$}
    \end{subfigure}
    \hfill
    \begin{subfigure}{0.49\textwidth}
        \centering
        \begin{tikzpicture}[->,>=stealth,shorten >=1pt,auto,node distance=2.5cm,scale=0.7,transform shape]

\tikzstyle{state}=[circle,draw,minimum size=1.1cm]
\def\weightpos{0.8};
\def\bend{5};
\def\bendsingle{0};
\def\zerolabel{\,};

\node[state] (t1) at (4,0) {$t_1$};
\node[state] (t2) at (8,0) {$t_2$};
\node[state] (t3) at (12,0) {$t_3$};

\node[state] (b1) at (4,-4) {$b_1$};
\node[state] (b2) at (8,-4) {$b_2$};
\node[state] (b3) at (12,-4) {$b_3$};

\path[->] (t1) edge[loop above] node {$13$} (t1);
\path[->] (t2) edge[loop above, inpolicy] node {$\mathbf{14}$} (t2);
\path[->] (t3) edge[loop above, inpolicy] node {$\mathbf{15}$} (t3);

\path[->] (t2) edge[pos=0.75, above, gray] node { } (t1);
\path[->] (t3) edge[pos=0.75, above, gray] node { } (t2);
\path[->] (t3) edge[bend right=18, pos=0.75, above, gray] node { } (t1);

\path[->] (t1) edge[bend right=\bend, pos=\weightpos, inpolicy] node[left] {$\mathbf{16}$} (b1);
\path[->] (t2) edge[bend right=\bend, pos=\weightpos] node[left] {$16$} (b2);
\path[->] (t2) edge[bend right=\bend, pos=\weightpos] node[left] {$16$} (b1);
\path[->] (t3) edge[bend right=\bend, pos=0.75] node[left] {$16$} (b1);
\path[->] (t3) edge[bend right=\bend, pos=\weightpos] node[left] {$16$} (b2);
\path[->] (t3) edge[bend right=\bend, pos=\weightpos] node[left] {$16$} (b3);

\path[->] (b1) edge[bend right=\bend,pos=\weightpos, gray] node[right] {\zerolabel} (t1);
\path[->] (b1) edge[bend right=\bend,pos=\weightpos, gray] node[right] {\zerolabel} (t2);
\path[->] (b1) edge[bend right=\bend,pos=\weightpos, inpolicy] node[right] {$\mathbf{0}$} (t3);
\path[->] (b2) edge[bend right=\bend,pos=\weightpos, gray] node[right] {\zerolabel} (t1);
\path[->] (b2) edge[bend right=\bend,pos=\weightpos, gray] node[right] {\zerolabel} (t2);
\path[->] (b2) edge[bend right=\bend,pos=\weightpos, inpolicy] node[right] {$\mathbf{0}$} (t3);
\path[->] (b3) edge[bend right=\bendsingle,pos=\weightpos, gray] node[right] {\zerolabel} (t1);
\path[->] (b3) edge[bend right=\bendsingle,pos=\weightpos, gray] node[right] {\zerolabel} (t2);
\path[->] (b3) edge[bend right=\bend,pos=\weightpos, inpolicy] node[right] {$\mathbf{0}$} (t3);

\path[->] (b2) edge[pos=0.6, below] node {$16$} (b1);
\path[->] (b3) edge[pos=0.6, below] node {$16$} (b2);
\path[->] (b3) edge[bend left = 18, pos=0.6, below] node {$16$} (b1);

\end{tikzpicture}
        \caption{Policy $\initpol_3$}
    \end{subfigure}
    \hfill
    \begin{subfigure}{0.49\textwidth}
        \centering
        \begin{tikzpicture}[->,>=stealth,shorten >=1pt,auto,node distance=2.5cm,scale=0.7,transform shape]

\tikzstyle{state}=[circle,draw,minimum size=1.1cm]
\def\weightpos{0.8};
\def\bend{5};
\def\bendsingle{0};
\def\zerolabel{\,};

\node[state] (t1) at (4,0) {$t_1$};
\node[state] (t2) at (8,0) {$t_2$};
\node[state] (t3) at (12,0) {$t_3$};

\node[state] (b1) at (4,-4) {$b_1$};
\node[state] (b2) at (8,-4) {$b_2$};
\node[state] (b3) at (12,-4) {$b_3$};

\path[->] (t1) edge[loop above] node {$13$} (t1);
\path[->] (t2) edge[loop above] node {$14$} (t2);
\path[->] (t3) edge[loop above, inpolicy] node {$\mathbf{15}$} (t3);

\path[->] (t2) edge[pos=0.75, above, gray] node { } (t1);
\path[->] (t3) edge[pos=0.75, above, gray] node { } (t2);
\path[->] (t3) edge[bend right=18, pos=0.75, above, gray] node { } (t1);

\path[->] (t1) edge[bend right=\bend, pos=\weightpos, inpolicy] node[left] {$\mathbf{16}$} (b1);
\path[->] (t2) edge[bend right=\bend, pos=\weightpos] node[left] {$16$} (b2);
\path[->] (t2) edge[bend right=\bend, pos=\weightpos, inpolicy] node[left] {$\mathbf{16}$} (b1);
\path[->] (t3) edge[bend right=\bend, pos=0.75] node[left] {$16$} (b1);
\path[->] (t3) edge[bend right=\bend, pos=\weightpos] node[left] {$16$} (b2);
\path[->] (t3) edge[bend right=\bend, pos=\weightpos] node[left] {$16$} (b3);

\path[->] (b1) edge[bend right=\bend,pos=\weightpos, gray] node[right] {\zerolabel} (t1);
\path[->] (b1) edge[bend right=\bend,pos=\weightpos, gray] node[right] {\zerolabel} (t2);
\path[->] (b1) edge[bend right=\bend,pos=\weightpos, inpolicy] node[right] {$\mathbf{0}$} (t3);
\path[->] (b2) edge[bend right=\bend,pos=\weightpos, gray] node[right] {\zerolabel} (t1);
\path[->] (b2) edge[bend right=\bend,pos=\weightpos, gray] node[right] {\zerolabel} (t2);
\path[->] (b2) edge[bend right=\bend,pos=\weightpos, gray] node[right] {\zerolabel} (t3);
\path[->] (b3) edge[bend right=\bendsingle,pos=\weightpos, gray] node[right] {\zerolabel} (t1);
\path[->] (b3) edge[bend right=\bendsingle,pos=\weightpos, gray] node[right] {\zerolabel} (t2);
\path[->] (b3) edge[bend right=\bend,pos=\weightpos, gray] node[right] {\zerolabel} (t3);

\path[->] (b2) edge[pos=0.6, below, inpolicy] node {$\mathbf{16}$} (b1);
\path[->] (b3) edge[pos=0.6, below] node {$16$} (b2);
\path[->] (b3) edge[bend left = 18, pos=0.6, below, inpolicy] node {$\mathbf{16}$} (b1);

\end{tikzpicture}
        \caption{Policy $\policy_{3,1}$}
    \end{subfigure}
    \hfill
    \begin{subfigure}{0.49\textwidth}
        \centering
        \begin{tikzpicture}[->,>=stealth,shorten >=1pt,auto,node distance=2.5cm,scale=0.7,transform shape]

\tikzstyle{state}=[circle,draw,minimum size=1.1cm]
\def\weightpos{0.8};
\def\bend{5};
\def\bendsingle{0};
\def\zerolabel{\,};

\node[state] (t1) at (4,0) {$t_1$};
\node[state] (t2) at (8,0) {$t_2$};
\node[state] (t3) at (12,0) {$t_3$};

\node[state] (b1) at (4,-4) {$b_1$};
\node[state] (b2) at (8,-4) {$b_2$};
\node[state] (b3) at (12,-4) {$b_3$};

\path[->] (t1) edge[loop above] node {$13$} (t1);
\path[->] (t2) edge[loop above] node {$14$} (t2);
\path[->] (t3) edge[loop above, inpolicy] node {$\mathbf{15}$} (t3);

\path[->] (t2) edge[pos=0.75, above, gray] node { } (t1);
\path[->] (t3) edge[pos=0.75, above, gray] node { } (t2);
\path[->] (t3) edge[bend right=18, pos=0.75, above, gray] node { } (t1);

\path[->] (t1) edge[bend right=\bend, pos=\weightpos, inpolicy] node[left] {$\mathbf{16}$} (b1);
\path[->] (t2) edge[bend right=\bend, pos=\weightpos, inpolicy] node[left] {$\mathbf{16}$} (b2);
\path[->] (t2) edge[bend right=\bend, pos=\weightpos] node[left] {$16$} (b1);
\path[->] (t3) edge[bend right=\bend, pos=0.75] node[left] {$16$} (b1);
\path[->] (t3) edge[bend right=\bend, pos=\weightpos] node[left] {$16$} (b2);
\path[->] (t3) edge[bend right=\bend, pos=\weightpos] node[left] {$16$} (b3);

\path[->] (b1) edge[bend right=\bend,pos=\weightpos, gray] node[right] {\zerolabel} (t1);
\path[->] (b1) edge[bend right=\bend,pos=\weightpos, gray] node[right] {\zerolabel} (t2);
\path[->] (b1) edge[bend right=\bend,pos=\weightpos, inpolicy] node[right] {$\mathbf{0}$} (t3);
\path[->] (b2) edge[bend right=\bend,pos=\weightpos, gray] node[right] {\zerolabel} (t1);
\path[->] (b2) edge[bend right=\bend,pos=\weightpos, gray] node[right] {\zerolabel} (t2);
\path[->] (b2) edge[bend right=\bend,pos=\weightpos, gray] node[right] {\zerolabel} (t3);
\path[->] (b3) edge[bend right=\bendsingle,pos=\weightpos, gray] node[right] {\zerolabel} (t1);
\path[->] (b3) edge[bend right=\bendsingle,pos=\weightpos, gray] node[right] {\zerolabel} (t2);
\path[->] (b3) edge[bend right=\bend,pos=\weightpos, gray] node[right] {\zerolabel} (t3);

\path[->] (b2) edge[pos=0.6, below, inpolicy] node {$\mathbf{16}$} (b1);
\path[->] (b3) edge[pos=0.6, below, inpolicy] node {$\mathbf{16}$} (b2);
\path[->] (b3) edge[bend left = 18, pos=0.6, below] node {$16$} (b1);

\end{tikzpicture}
        \caption{Policy $\policy_{3,2}$}
    \end{subfigure}
    \caption{Part II: the sequence of policies appearing in Howard's policy iteration over our running example. Thick lines correspond to policy choices. Unlabeled (gray) edges have weight 0.}
    \label{fig:policy-sequence-2}
\end{figure*}
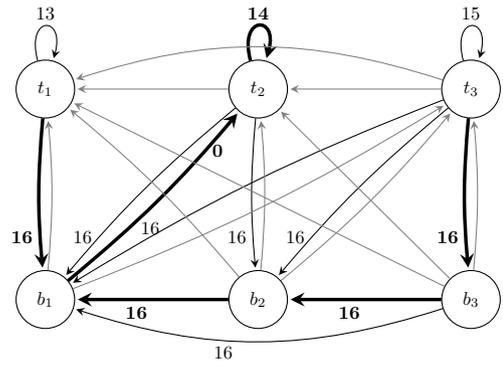
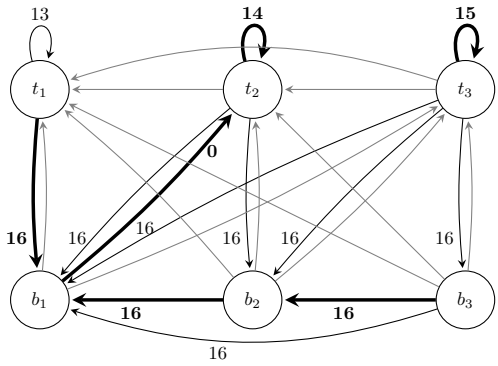

\newpage
\section{The policies in the general case}
\label{sec:general-pol}
We illustrate the sequence of policies that appear from policy $\initpol_i$ to policy $\initpol_{i+1}$. To keep the display clear, we omit edge weights and edges not in the policy from the figures and show only the edges in the current policy.
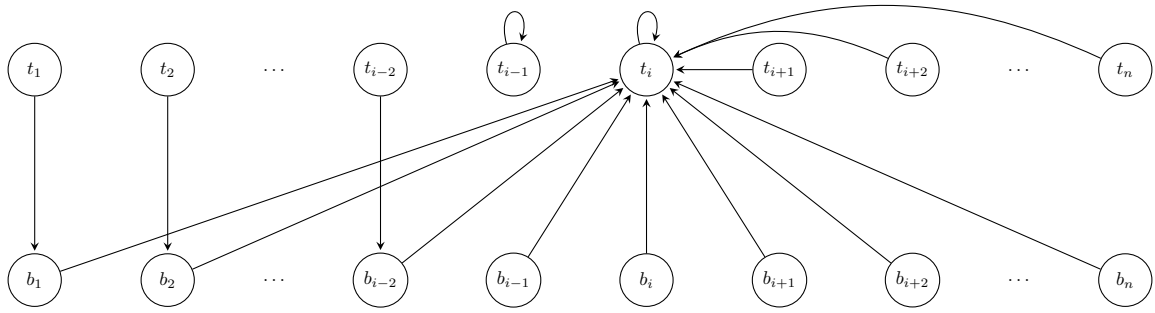
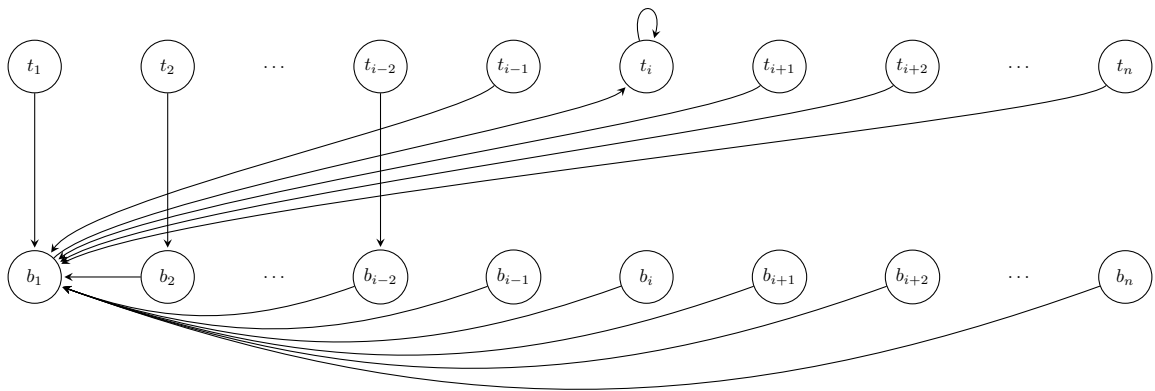
\begin{figure*}[h]
    \centering
    \begin{subfigure}{0.99\textwidth}
        \centering
        \begin{tikzpicture}[->,>=stealth,shorten >=1pt,auto,node distance=3.5cm,scale=0.7,transform shape]

    \tikzstyle{state}=[circle,draw,minimum size=01cm]
    \def\weightpos{0.8};
    \def\bend{5};
    \def\bendsingle{0};
    \def\zerolabel{\,};
    
    \node[state] (t1) at (0,0) {$t_1$};
    \node[state] (t2) at (2.5,0) {$t_2$};

    \node (dots1) at (4.5,0) {$\ldots$};
    
    \node[state] (t3) at (6.5,0) {$t_{i-2}$};
    \node[state] (t4) at (9,0) {$t_{i-1}$};

    \node[state] (t5) at (11.5,0) {$t_{i}$};
    \node[state] (t6) at (14,0) {$t_{i+1}$};
    \node[state] (t7) at (16.5,0) {$t_{i+2}$};

    \node (dots2) at (18.5,0) {$\ldots$};
    
    \node[state] (t8) at (20.5,0) {$t_n$};
    
    \node[state] (b1) at (0,-4) {$b_1$};
    \node[state] (b2) at (2.5,-4) {$b_2$};
    
    \node (dots3) at (4.5,-4) {$\ldots$};
    
    \node[state] (b3) at (6.5,-4) {$b_{i-2}$};
    \node[state] (b4) at (9,-4) {$b_{i-1}$};
    \node[state] (b5) at (11.5,-4) {$b_{i}$};
    \node[state] (b6) at (14,-4) {$b_{i+1}$};
    \node[state] (b7) at (16.5,-4) {$b_{i+2}$};
    
    \node (dots4) at (18.5,-4) {$\ldots$};
    
    \node[state] (b8) at (20.5,-4) {$b_n$};

    \path[->] (t4) edge[loop above] node {} (t4);
    \path[->] (t5) edge[loop above] node {} (t5);
    
    \path[->] (t1) edge (b1);
    \path[->] (t2) edge (b2);
    \path[->] (t3) edge (b3);
    
    \path[->] (b1) edge[bend right=0] (t5);
    \path[->] (b2) edge[bend right=0] (t5);
    \path[->] (b3) edge[bend right=0] (t5);
    \path[->] (b4) edge[bend right=0] (t5);
    \path[->] (b5) edge[bend right=0] (t5);
    \path[->] (b6) edge[bend right=0] (t5);
    \path[->] (b7) edge[bend right=0] (t5);
    \path[->] (b8) edge[bend right=0] (t5);

    \path[->] (t6) edge[bend right=0] (t5);
    \path[->] (t7) edge[bend right=25] (t5);
    \path[->] (t8) edge[bend right=25] (t5);

\end{tikzpicture}
    
        \caption{Policy $\initpol_{i}$}
    \end{subfigure} 
    \par\bigskip
    \begin{subfigure}{0.99\textwidth}
        \centering
        \begin{tikzpicture}[->,>=stealth,shorten >=1pt,auto,node distance=3.5cm,scale=0.7,transform shape]

    \tikzstyle{state}=[circle,draw,minimum size=01cm]
    \def\weightpos{0.8};
    \def\bend{5};
    \def\bendsingle{0};
    \def\zerolabel{\,};
    
    \node[state] (t1) at (0,0) {$t_1$};
    \node[state] (t2) at (2.5,0) {$t_2$};

    \node (dots1) at (4.5,0) {$\ldots$};
    
    \node[state] (t3) at (6.5,0) {$t_{i-2}$};
    \node[state] (t4) at (9,0) {$t_{i-1}$};

    \node[state] (t5) at (11.5,0) {$t_{i}$};
    \node[state] (t6) at (14,0) {$t_{i+1}$};
    \node[state] (t7) at (16.5,0) {$t_{i+2}$};

    \node (dots2) at (18.5,0) {$\ldots$};
    
    \node[state] (t8) at (20.5,0) {$t_n$};
    
    \node[state] (b1) at (0,-4) {$b_1$};
    \node[state] (b2) at (2.5,-4) {$b_2$};
    
    \node (dots3) at (4.5,-4) {$\ldots$};
    
    \node[state] (b3) at (6.5,-4) {$b_{i-2}$};
    \node[state] (b4) at (9,-4) {$b_{i-1}$};
    \node[state] (b5) at (11.5,-4) {$b_{i}$};
    \node[state] (b6) at (14,-4) {$b_{i+1}$};
    \node[state] (b7) at (16.5,-4) {$b_{i+2}$};
    
    \node (dots4) at (18.5,-4) {$\ldots$};
    
    \node[state] (b8) at (20.5,-4) {$b_n$};

    \path[->] (t5) edge[loop above] node {} (t5);
    
    \path[->] (t1) edge (b1);
    \path[->] (t2) edge (b2);
    \path[->] (t3) edge (b3);

    \path[->] (b2) edge (b1);
    \path[->] (b3) edge[bend left=20] (b1);
    \path[->] (b4) edge[bend left=20] (b1);
    \path[->] (b5) edge[bend left=20] (b1);
    \path[->] (b6) edge[bend left=20] (b1);
    \path[->] (b7) edge[bend left=20] (b1);
    \path[->] (b8) edge[bend left=20] (b1);
    
    \draw[->] (t4) .. controls ++(-1,-1) and ++(1,1.5) .. (b1);
    \draw[->] (t6) .. controls ++(-1,-1) and ++(2,1.5) .. (b1);
    \draw[->] (t7) .. controls ++(-1,-1) and ++(2.5,1.5) .. (b1);
    \draw[->] (t8) .. controls ++(-1,-1) and ++(3,1.5) .. (b1);

    \draw[->] (b1) .. controls ++(1.5,1.5) and ++(-1,-1) .. (t5);

\end{tikzpicture}
    
        \caption{Policy $\policy_{i,1}$.}
    \end{subfigure}
    \caption{Part I: policies appearing in Howard's policy iteration on $P_n$. Only edges in the policy are shown; edge weights are omitted.}
\end{figure*}

\begin{figure*}[]\ContinuedFloat
    \centering
    \rotatebox{90}{
        \begin{minipage}{0.85\textheight}  
            \centering
            \begin{subfigure}{0.8\textheight}
                \centering
                \begin{tikzpicture}[->,>=stealth,shorten >=1pt,auto,node distance=3.5cm,scale=0.7,transform shape]

    \tikzstyle{state}=[circle,draw,minimum size=01cm]
    \def\weightpos{0.8};
    \def\bend{5};
    \def\bendsingle{0};
    \def\zerolabel{\,};
    
    \node[state] (t1) at (-2,0) {$t_1$};
    \node[state] (t2) at (0.5,0) {$t_2$};

    \node (dots1) at (2.5,0) {$\ldots$};
    
    \node[state] (t3) at (4.5,0) {$t_{j}$};
    \node[state] (t4) at (7,0) {$t_{j+1}$};
    \node[state] (t5) at (9.5,0) {$t_{j+2}$};

    \node (dots5) at (11.5,0) {$\ldots$};

    \node[state] (t6) at (13.5,0) {$t_{i-1}$};
    \node[state] (t7) at (16,0) {$t_{i}$};
    \node[state] (t8) at (18.5,0) {$t_{i+1}$};

    \node (dots2) at (20.5,0) {$\ldots$};
    
    \node[state] (t9) at (22.5,0) {$t_n$};
    
    \node[state] (b1) at (-2,-4) {$b_1$};
    \node[state] (b2) at (0.5,-4) {$b_2$};
    
    \node (dots3) at (2.5,-4) {$\ldots$};
    
    \node[state] (b3) at (4.5,-4) {$b_{j}$};
    \node[state] (b4) at (7,-4) {$b_{j+1}$};
    \node[state] (b5) at (9.5,-4) {$b_{j+2}$};

    \node (dots6) at (11.5,-4) {$\ldots$};
    
    \node[state] (b6) at (13.5,-4) {$b_{i-1}$};
    \node[state] (b7) at (16,-4) {$b_{i}$};
    \node[state] (b8) at (18.5,-4) {$b_{i+1}$};
    
    \node (dots4) at (20.5,-4) {$\ldots$};
    
    \node[state] (b9) at (22.5,-4) {$b_n$};

    \path[->] (t7) edge[loop above] node {} (t7);
 
    \path[->] (t1) edge (b1);
    \path[->] (t2) edge (b2);
    \path[->] (t3) edge (b3);

       \draw[->] (t4) .. controls ++(-1,-2) and ++(1,1.5) .. (b3);
       \draw[->] (t5) .. controls ++(-1,-1) and ++(1.5,1.5) .. (b3);
       \draw[->] (t6) .. controls ++(-1,-1) and ++(2,1.5) .. (b3);
       \draw[->] (t8) .. controls ++(-1,-1) and ++(3,1.5) .. (b3);
       \draw[->] (t9) .. controls ++(-1,-1) and ++(3.5,1.5) .. (b3); 

    \draw[->] (b1) .. controls ++(2,2) and ++(-4,-2) .. (t7);
    
    \path[->] (b2) edge (b1);
    \path[->] (b4) edge (b3);
    
    \draw[-] (b3) -- +(-1.2,0);
    \node at (2.5,-4) {$\ldots$};
    \draw[->] ([xshift=0.7cm]b2.east) -- (b2);
    
    \path[->] (b5) edge[bend left=25] (b3);
    \path[->] (b6) edge[bend left=25] (b3);
    \path[->] (b7) edge[bend left=25] (b3);
    \path[->] (b8) edge[bend left=25] (b3);
    \path[->] (b9) edge[bend left=25] (b3);
    
\end{tikzpicture}
    
                \caption{Policy $\policy_{i,j}$}
            \end{subfigure}
            \par\bigskip
            \begin{subfigure}{0.8\textheight}
                \centering
                \begin{tikzpicture}[->,>=stealth,shorten >=1pt,auto,node distance=3.5cm,scale=0.7,transform shape]

    \tikzstyle{state}=[circle,draw,minimum size=01cm]
    \def\weightpos{0.8};
    \def\bend{5};
    \def\bendsingle{0};
    \def\zerolabel{\,};
    
    \node[state] (t1) at (-2,0) {$t_1$};
    \node[state] (t2) at (0.5,0) {$t_2$};

    \node (dots1) at (2.5,0) {$\ldots$};
    
    \node[state] (t3) at (4.5,0) {$t_{j}$};
    \node[state] (t4) at (7,0) {$t_{j+1}$};
    \node[state] (t5) at (9.5,0) {$t_{j+2}$};

    \node (dots5) at (11.5,0) {$\ldots$};

    \node[state] (t6) at (13.5,0) {$t_{i-1}$};
    \node[state] (t7) at (16,0) {$t_{i}$};
    \node[state] (t8) at (18.5,0) {$t_{i+1}$};

    \node (dots2) at (20.5,0) {$\ldots$};
    
    \node[state] (t9) at (22.5,0) {$t_n$};
    
    \node[state] (b1) at (-2,-4) {$b_1$};
    \node[state] (b2) at (0.5,-4) {$b_2$};
    
    \node (dots3) at (2.5,-4) {$\ldots$};
    
    \node[state] (b3) at (4.5,-4) {$b_{j}$};
    \node[state] (b4) at (7,-4) {$b_{j+1}$};
    \node[state] (b5) at (9.5,-4) {$b_{j+2}$};

    \node (dots6) at (11.5,-4) {$\ldots$};
    
    \node[state] (b6) at (13.5,-4) {$b_{i-1}$};
    \node[state] (b7) at (16,-4) {$b_{i}$};
    \node[state] (b8) at (18.5,-4) {$b_{i+1}$};
    
    \node (dots4) at (20.5,-4) {$\ldots$};
    
    \node[state] (b9) at (22.5,-4) {$b_n$};

    \path[->] (t7) edge[loop above] node {} (t7);
 
    \path[->] (t1) edge (b1);
    \path[->] (t2) edge (b2);
    \path[->] (t3) edge (b3);
    \path[->, inpolicy] (t4) edge (b4);
    
       \draw[->, inpolicy] (t5) .. controls ++(-1,-2) and ++(1,1.5) .. (b4);
       \draw[->, inpolicy] (t6) .. controls ++(-1,-1) and ++(1.5,1.5) .. (b4);
       \draw[->, inpolicy] (t8) .. controls ++(-1,-1) and ++(2.5,1.5) .. (b4);
       \draw[->, inpolicy] (t9) .. controls ++(-1,-1) and ++(3,1.5) .. (b4);

    \draw[->] (b1) .. controls ++(2,2) and ++(-4,-2) .. (t7);
    
    \path[->] (b2) edge (b1);
    \path[->] (b4) edge (b3);
    \path[->, inpolicy] (b5) edge (b4);
    
    \draw[-] (b3) -- +(-1.2,0);
    \node at (2.5,-4) {$\ldots$};
    \draw[->] ([xshift=0.7cm]b2.east) -- (b2);
    
    \path[->] (b6) edge[bend left=25, inpolicy] (b4);
    \path[->] (b7) edge[bend left=25, inpolicy] (b4);
    \path[->] (b8) edge[bend left=25, inpolicy] (b4);
    \path[->] (b9) edge[bend left=25, inpolicy] (b4);
    
\end{tikzpicture}
    
                \caption{Policy $\policy_{i,j+1}$}
            \end{subfigure}
        \end{minipage}
    }
    \caption{Part II: Comparison of policies $\policy_{i,j}$ and $\policy_{i,j+1}$, with $j>1$, appearing in Howard's policy iteration on $P_n$. Only edges in the policy are shown; edge weights are omitted. The edges in $\policy_{i,j+1}$ that differ from $\policy_{i,j}$ are bold.}
    
\end{figure*}
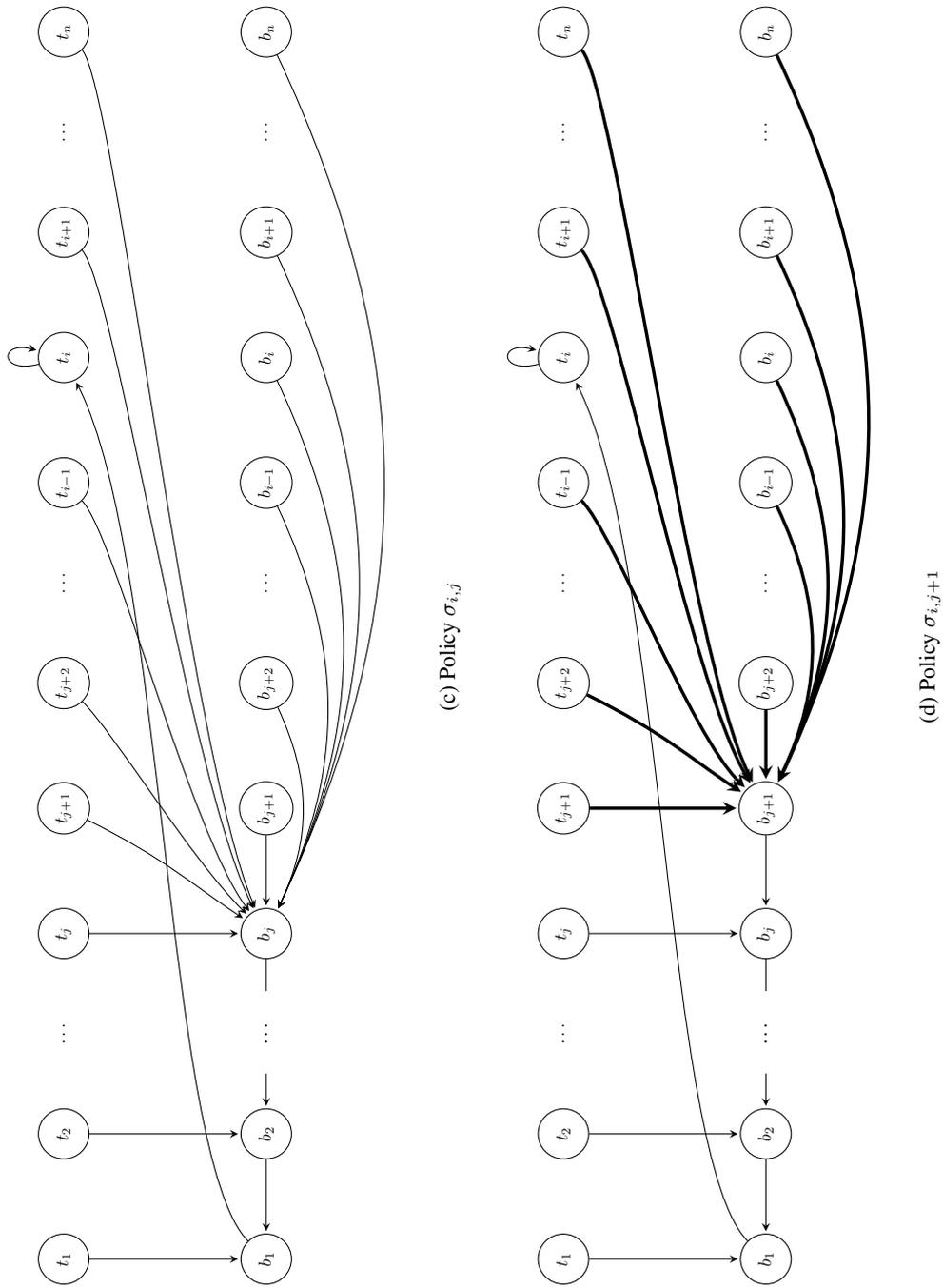

\begin{figure*}[h]\ContinuedFloat
\centering
    \begin{subfigure}{0.99\textwidth}
        \centering
        \begin{tikzpicture}[->,>=stealth,shorten >=1pt,auto,node distance=3.5cm,scale=0.7,transform shape]

    \tikzstyle{state}=[circle,draw,minimum size=01cm]
    \def\weightpos{0.8};
    \def\bend{5};
    \def\bendsingle{0};
    \def\zerolabel{\,};
    
    \node[state] (t1) at (0,0) {$t_1$};
    \node[state] (t2) at (2.5,0) {$t_2$};

    \node (dots1) at (4.5,0) {$\ldots$};
    
    \node[state] (t3) at (6.5,0) {$t_{i-1}$};
    \node[state] (t4) at (9,0) {$t_i$};

    \node[state] (t5) at (11.5,0) {$t_{i+1}$};
    \node[state] (t6) at (14,0) {$t_{i+2}$};
    \node[state] (t7) at (16.5,0) {$t_{i+3}$};

    \node (dots2) at (18.5,0) {$\ldots$};
    
    \node[state] (t8) at (20.5,0) {$t_n$};
    
    \node[state] (b1) at (0,-4) {$b_1$};
    \node[state] (b2) at (2.5,-4) {$b_2$};
    
    \node (dots3) at (4.5,-4) {$\ldots$};
    
    \node[state] (b3) at (6.5,-4) {$b_{i-1}$};
    \node[state] (b4) at (9,-4) {$b_i$};
    \node[state] (b5) at (11.5,-4) {$b_{i+1}$};
    \node[state] (b6) at (14,-4) {$b_{i+2}$};
    \node[state] (b7) at (16.5,-4) {$b_{i+3}$};
    
    \node (dots4) at (18.5,-4) {$\ldots$};
    
    \node[state] (b8) at (20.5,-4) {$b_n$};

    \path[->] (t4) edge[loop above] node {} (t4);
    
    \path[->] (t1) edge (b1);
    \path[->] (t2) edge (b2);
    \path[->] (t3) edge (b3);
    \path[->] (t5) edge (b5);
    \path[->] (t6) edge (b5);
    \path[->] (t7) edge (b5);
    \path[->] (t8) edge (b5);
    
    \path[->] (b1) edge[bend right=0] (t4);
    
    \path[->] (b2) edge (b1);
    \path[->] (b4) edge (b3);
    \path[->] (b5) edge (b4);
    \path[->] (b6) edge (b5);
    
    \draw[-] (b3) -- +(-1.2,0);
    \node at (4.5,-4) {$\ldots$};
    \draw[->] ([xshift=0.7cm]b2.east) -- (b2);
    
    \path[->] (b7) edge[bend left=25] (b5);
    \path[->] (b8) edge[bend left=30] (b5);

\end{tikzpicture}
    
        \caption{Policy $\policy_{i,i+1}$}
    \end{subfigure} 
    \par\bigskip
    \begin{subfigure}{0.99\textwidth}
        \centering
        \begin{tikzpicture}[->,>=stealth,shorten >=1pt,auto,node distance=3.5cm,scale=0.7,transform shape]

    \tikzstyle{state}=[circle,draw,minimum size=01cm]
    \def\weightpos{0.8};
    \def\bend{5};
    \def\bendsingle{0};
    \def\zerolabel{\,};
    
    \node[state] (t1) at (0,0) {$t_1$};
    \node[state] (t2) at (2.5,0) {$t_2$};

    \node (dots1) at (4.5,0) {$\ldots$};
    
    \node[state] (t3) at (6.5,0) {$t_{i-1}$};
    \node[state] (t4) at (9,0) {$t_i$};

    \node[state] (t5) at (11.5,0) {$t_{i+1}$};
    \node[state] (t6) at (14,0) {$t_{i+2}$};
    \node[state] (t7) at (16.5,0) {$t_{i+3}$};

    \node (dots2) at (18.5,0) {$\ldots$};
    
    \node[state] (t8) at (20.5,0) {$t_n$};
    
    \node[state] (b1) at (0,-4) {$b_1$};
    \node[state] (b2) at (2.5,-4) {$b_2$};
    
    \node (dots3) at (4.5,-4) {$\ldots$};
    
    \node[state] (b3) at (6.5,-4) {$b_{i-1}$};
    \node[state] (b4) at (9,-4) {$b_i$};
    \node[state] (b5) at (11.5,-4) {$b_{i+1}$};
    \node[state] (b6) at (14,-4) {$b_{i+2}$};
    \node[state] (b7) at (16.5,-4) {$b_{i+3}$};
    
    \node (dots4) at (18.5,-4) {$\ldots$};
    
    \node[state] (b8) at (20.5,-4) {$b_n$};

    \path[->] (t4) edge[loop above] node {} (t4);
    \path[->] (t5) edge[loop above] node {} (t5);
    
    \path[->] (t1) edge (b1);
    \path[->] (t2) edge (b2);
    \path[->] (t3) edge (b3);
    \path[->] (t6) edge (b6);
    \path[->] (t7) edge (b6);
    \path[->] (t8) edge (b6);
    
    \path[->] (b1) edge[bend right=0] (t4);
    
    \path[->] (b2) edge (b1);
    \path[->] (b4) edge (b3);
    \path[->] (b5) edge (b4);
    \path[->] (b6) edge (b5);
    \path[->] (b7) edge (b6);
    
    \draw[-] (b3) -- +(-1.2,0);
    \node at (4.5,-4) {$\ldots$};
    \draw[->] ([xshift=0.7cm]b2.east) -- (b2);
    
    \path[->] (b8) edge[bend left=20] (b6);
    
\end{tikzpicture}
    
        \caption{Policy $\finalpol_{i}$}
    \end{subfigure}
    \par\bigskip
    \begin{subfigure}{0.99\textwidth}
         \centering
        \begin{tikzpicture}[->,>=stealth,shorten >=1pt,auto,node distance=3.5cm,scale=0.7,transform shape]

    \tikzstyle{state}=[circle,draw,minimum size=01cm]
    \def\weightpos{0.8};
    \def\bend{5};
    \def\bendsingle{0};
    \def\zerolabel{\,};
    
    \node[state] (t1) at (0,0) {$t_1$};
    \node[state] (t2) at (2.5,0) {$t_2$};

    \node (dots1) at (4.5,0) {$\ldots$};
    
    \node[state] (t3) at (6.5,0) {$t_{i-1}$};
    \node[state] (t4) at (9,0) {$t_i$};

    \node[state] (t5) at (11.5,0) {$t_{i+1}$};
    \node[state] (t6) at (14,0) {$t_{i+2}$};
    \node[state] (t7) at (16.5,0) {$t_{i+3}$};

    \node (dots2) at (18.5,0) {$\ldots$};
    
    \node[state] (t8) at (20.5,0) {$t_n$};
    
    \node[state] (b1) at (0,-4) {$b_1$};
    \node[state] (b2) at (2.5,-4) {$b_2$};
    
    \node (dots3) at (4.5,-4) {$\ldots$};
    
    \node[state] (b3) at (6.5,-4) {$b_{i-1}$};
    \node[state] (b4) at (9,-4) {$b_i$};
    \node[state] (b5) at (11.5,-4) {$b_{i+1}$};
    \node[state] (b6) at (14,-4) {$b_{i+2}$};
    \node[state] (b7) at (16.5,-4) {$b_{i+3}$};
    
    \node (dots4) at (18.5,-4) {$\ldots$};
    
    \node[state] (b8) at (20.5,-4) {$b_n$};

    \path[->] (t4) edge[loop above] node {} (t4);
    \path[->] (t5) edge[loop above] node {} (t5);
    
    \path[->] (t1) edge (b1);
    \path[->] (t2) edge (b2);
    \path[->] (t3) edge (b3);
    
    \path[->] (b1) edge[bend right=0] (t5);
    \path[->] (b2) edge[bend right=0] (t5);
    \path[->] (b3) edge[bend right=0] (t5);
    \path[->] (b4) edge[bend right=0] (t5);
    \path[->] (b5) edge[bend right=0] (t5);
    \path[->] (b6) edge[bend right=0] (t5);
    \path[->] (b7) edge[bend right=0] (t5);
    \path[->] (b8) edge[bend right=0] (t5);

    \path[->] (t6) edge[bend right=0] (t5);
    \path[->] (t7) edge[bend right=25] (t5);
    \path[->] (t8) edge[bend right=25] (t5);

\end{tikzpicture}
    
        \caption{Policy $\initpol_{i+1}$}
    \end{subfigure}
    \caption{Part III: policies appearing in Howard's policy iteration on $P_n$. Only edges in the policy are shown; edge weights are omitted.}
\end{figure*}

\end{document}